\documentclass[11pt,oneside]{amsart}
\usepackage{amssymb}
\usepackage{amsfonts}
\usepackage{amsmath}
\usepackage{amsthm}
\usepackage{dsfont}
\usepackage{mathtools}
\usepackage{geometry}
\usepackage[colorlinks=true,linkcolor=blue,urlcolor=blue, hypertexnames=false]{hyperref}

\mathtoolsset{showonlyrefs}
\usepackage{amssymb}
\usepackage{amsfonts}
\usepackage{amsmath}
\usepackage{amsthm}
\usepackage{dsfont}
\usepackage{xcolor}
\usepackage{float}
\usepackage{comment}
\usepackage{enumerate}
\usepackage{bbm}
\usepackage{bm}

\newtheorem{theorem}{Theorem}
\newtheorem{lemma}[theorem]{Lemma}

\theoremstyle{definition} 

\newtheorem{remark}{Remark}
\newtheorem{assumption}{Assumption}
\newtheorem{example}{Example}
\newtheorem{problem}{Problem}

\newcommand{\ud}{\mathrm d}

\newcommand{\cF}{\mathcal F}

\newcommand{\cL}{\mathcal L}
\newcommand{\cT}{\mathcal T}
\newcommand{\cU}{\mathcal U}
\newcommand{\cP}{\mathcal P}

\def\bx{\bm{x}} 
\def\bX{\bm{X}}  
\def\R{{\mathbb R}}

\def\P{{\mathsf P}}
\def\Q{{\mathsf Q}}
\def\E{{\mathsf E}}
\def\C{{\mathbb C}}

\renewcommand{\mid}{\,|\,}
\newcommand{\Law}{\mathcal{L}aw}
\newcommand{\KL}{\mathcal{D}_{\mathrm{KL}}}

\newcommand{\ind}{\mathds{1}}
\def\Appx{Appendix} 
\def\subT{0 \leq t \leq T}
\def\muT{\mu_{\mathrm{obj}}}
\def\muR{\mu_{\mathrm{ref}}}
\def\DT{\mathcal{D}_{\mathrm{obj}}}
\def\DR{\mathcal{D}_{\mathrm{ref}}}
\def\TGT{\mathrm{obj}}
\def\REF{\mathrm{ref}} 
\numberwithin{equation}{section}
\usepackage{natbib}

\title{Soft-constrained Schr\"{o}dinger Bridge: a Stochastic Control Approach} 
\author[J.\ Garg, X.\ Zhang, Q.\ Zhou]{Jhanvi Garg$^{*}$, Xianyang Zhang$^{*}$, Quan Zhou$^{*, \dagger}$}

\thanks{$^{*}$Department of Statistics, Texas A\&M University, College Station, TX 77843, USA}
\thanks{$^{\dagger}$Corresponding author (email: quan@stat.tamu.edu)}

\begin{document}
\begin{abstract}
Schr\"{o}dinger bridge can be viewed as a continuous-time stochastic control problem where the goal is to find an optimally controlled diffusion process whose terminal distribution coincides with a pre-specified target distribution.  
We propose to generalize this problem by allowing the terminal distribution to differ from the target but penalizing the Kullback-Leibler divergence between the two distributions. 
We call this new control problem soft-constrained Schr\"{o}dinger bridge (SSB). 
The main contribution of this work is a theoretical derivation of the solution to SSB, which shows that the terminal distribution of the optimally controlled process is a geometric mixture of the target and some other distribution.  
This result is further extended to a time series setting. 
One application  is the development of robust generative diffusion models. We propose  a score matching-based algorithm for sampling from geometric mixtures and showcase its use via a numerical example for the MNIST data set.
\end{abstract}
\maketitle

\section{Introduction}\label{sec:intro}

\subsection{Schr\"{o}dinger Bridge and Its Applications}
Let $X = (X_t)_{\subT}$ be a diffusion process over the finite time interval $[0, T]$ with initial distribution $\mu_0$. 
Schr\"{o}dinger bridge seeks an optimal steering of $X$ towards a pre-specified terminal distribution $\mu_T$ such that the resulting controlled process is closest to $X$ in terms of Kullback-Leibler (KL) divergence~\citep{schrodinger1931umkehrung, schrodinger1932theorie}. 
Under certain regularity conditions, the optimally controlled process is another diffusion with the same diffusion coefficients as $X$ but an additional drift term. 
This result  has been obtained via different approaches and at varying levels of generality, and among the seminal works are~\cite{fortet1940resolution, beurling1960automorphism, jamison1975markov, follmer1988random, dai1991stochastic}. For comprehensive reviews detailing the historical development, we refer readers to~\citet{leonard2013survey} and~\citet{chen2021stochastic}. 

The recent generative modeling literature has seen a surge in the use of Schr\"{o}dinger bridge. 
In these applications, $\mu_0$ is typically some distribution that is easy to sample from, and $\mu_T$ is the unknown distribution of a given data set. 
By numerically approximating the solution to the Schr\"{o}dinger bridge problem, one can generate samples from $\mu_T$ (i.e., synthetic data points that resemble the original data set). 
One such algorithm is presented by~\citet{debortoli2023diffusion} and~\citet{vargas2021solving}, who proposed to calculate the Schr\"{o}dinger bridge by approximating the iterative proportional fitting procedure~\citep{deming1940least} (the method of~\citet{debortoli2023diffusion} estimates the drift using score matching and neural networks while that of~\citet{vargas2021solving} uses maximum likelihood and Gaussian processes). 
Concurrently, \citet{pmlr-v139-wang21l} developed a two-stage method where an auxiliary Schr\"{o}dinger bridge is run first to generate samples from a smoothed version of $\mu_T$, and the second Schr\"{o}dinger bridge transports these samples towards $\mu_T$.   
Both approaches generalize the denoising diffusion model methods of~\citet{ho2020denoising} and~\citet{song2021scorebased}. 
Some other recent developments in this area include~\citet{chen2021likelihood, song2022applying,  peluchetti2023diffusion, richter2023improved, winkler2023score, tsb, vargas2024transport}. 

Though not the focus of this work, Schr\"{o}dinger bridge sampling methods can also be used when  samples from $\mu_T$ are not available but $\mu_T$ is known up to a normalizing constant;
see, e.g.,~\citet{huang2021schrodinger, zhang2021path, vargas2022denoising}, and see~\citet{heng2024diffusion} for a recent review. 
For the connections between Schr\"{o}dinger bridge, optimal transport and variational inference, see, e.g.,~\citet{chen2016relation, chen2021optimal,tzen2019theoretical}.

\subsection{Overview of This Work} 
The main contribution of this paper is the theoretical development of a generalized Schr\"{o}dinger bridge problem, which we call soft-constrained Schr\"{o}dinger bridge (SSB). 
We take the stochastic control approach that was employed by~\citet{mikami1990variational, dai1991stochastic, pra1990markov} for studying the original Schr\"{o}dinger bridge problem (see Problem~\ref{problem0}). 
In SSB, the terminal distribution of the controlled process does not need to precisely match $\mu_T$ but needs to be close to $\mu_T$ in terms of KL divergence. 
Formally, SSB differs from the original problem in that we replace the hard constraint on the terminal distribution with an additional cost term, parameterized by $\beta$, in the objective function to be minimized (see Problem~\ref{problem1}). 
A larger $\beta$ forces the terminal distribution of the controlled process to be closer to $\mu_T$. 
We rigorously find the solution to SSB and the expression for the drift term of the optimally controlled process.  
We show that SSB generalizes Schr\"{o}dinger bridge in the sense that as $\beta \to \infty$, the solution to the former coincides with the solution to the latter. 
An important implication of our results is that the terminal distribution of the controlled process should be a geometric mixture of $\mu_T$ and some other distribution; when $\mu_0$ is a Dirac measure, the other distribution is $\Law(X_T)$ (i.e., the terminal distribution of the uncontrolled process).  
We further extend our results to a time series generalization of SSB, where we are interested in modifying the joint distribution of $(X_{t_1}, \dots, X_{t_N})$ for $0 < t_1 < \cdots < t_N = T$.  
 
SSB can be used as a theoretical foundation for developing more flexible and robust sampling methods. 
First, when the KL divergence between $\mu_T$ and $\Law(X_T)$ is infinite, Schr\"{o}dinger bridge does not admit a solution, while SSB always does. A toy example illustrating the  consequences of this result is given in Example~\ref{ex:cauchy}.
More importantly, $\beta$ acts as a regularization parameter preventing the algorithm from overfitting to  $\mu_T$, which is crucial for some generative modeling tasks such as fine-tuning  with limited data~\citep{moon2022fine}.  
In such applications, $\mu_T$ contains information from a small or noisy data set, and one wants to improve the sample quality by harnessing knowledge from a large high-quality reference data set. To achieve this, we can train the uncontrolled process $X$ in SSB using the reference data set and then tune the value of $\beta$. 
We present a simple normal mixture example illustrating the effect of $\beta$ (see Example~\ref{ex:normal.mix}). 
For a more realistic example in generative modeling of images,  we use the MNIST data set and consider the task of generating new images of digit 8. We assume that the training data set only has 50 noisy images of digit 8, but we can use the data set of all the other digits as reference. 
As suggested by our theoretical findings, we can train a Schr\"{o}dinger bridge targeting a geometric mixture of the distributions of the two data sets. 
Such a Schr\"{o}dinger bridge cannot be learned by existing methods, and to address this, we propose a new score matching algorithm that utilizes importance sampling. 
We show that this approach yields high-quality images of digit 8 when $\beta$ is properly chosen.

The paper is structured as follows. In Section \ref{sec:relax.SB}, we present the stochastic control formulation of the SSB problem, and we derive its solution when $\mu_0$ is a Dirac measure. 
The solution to SSB for general initial conditions is obtained in Section \ref{sec:sol.RSB}, which involves solving a generalized Schr\"{o}dinger system. Section \ref{sec:tSB} extends the results  to the time series setting. In Section \ref{sec:sim}, we present a new algorithm for robust generative modeling and demonstrate its use via the MNIST data set. Proofs and auxiliary results are deferred to \Appx{}.

\subsection{Related Literature}
Our development of SSB builds upon the work of~\citet{dai1991stochastic}, which formulates Schr\"{o}dinger bridge as a stochastic control problem and derives the solution using the logarithmic transformation technique pioneered by Fleming~\citep{fleming1977exit, fleming2005logarithmic, fleming2012deterministic} and the result of~\citet{jamison1975markov}. 
The time series SSB problem is a generalization of the work of~\citet{tsb}, who extended the original Schr\"{o}dinger bridge problem to the time series setting but only considered the special case where $\mu_0$ is a Dirac measure. 
\citet{pavon1991free, blaquiere1992controllability} adopted an alternative stochastic control approach to studying Schr\"{o}dinger bridge, which was rooted in the same logarithmic transformation and also considered in~\citet{tzen2019theoretical, berner2022optimal}.  This approach can be applied to the SSB problem as well, but it requires the use of verification theorem.  
 
Motivated by robust network routing, a discrete version of the SSB problem was proposed and solved in~\cite{chen2019relaxed}, where $X$ is a discrete-time non-homogeneous Markov chain with finite state space.  The techniques used in this paper are very different, and to our knowledge, the continuous-time SSB problem has not been addressed in the literature.

\section{Problem Formulation}\label{sec:relax.SB}
Let $\mu_0, \mu_T$ be two probability distributions on $\R^d$ such that $\int x^2 \mu_0 (\ud x) < \infty$ and  $\mu_T \ll \lambda$, where $\lambda$ denotes the Lebesgue measure. Denote the density of $\mu_T$ by $f_T = \ud \mu_T / \ud \lambda$. 
Let $(\Omega, \mathcal{F}, \P)$ be a probability space, on which we define a standard $d$-dimensional Brownian motion $W = ( W_t )_{t \geq 0}$ and a random vector $\xi$ that is independent of $W$ and has distribution $\mu_0$. 
We will always use $X = (X_t)_{\subT}$ to denote a weak solution to the following stochastic differential equation (SDE) 
\begin{equation}\label{eq:udp} 
\begin{aligned}
X_0 = \xi, \text{ and } \ud X_t= b(X_t, t) \ud t + &\sigma \ud W_t
\end{aligned}
\end{equation} 
for  $t \in [0, T]$ ,where $b \colon \R^d \times [0, T] \rightarrow \R^d$ and $\sigma \in (0, \infty)$.   
Given a control $u = (u_t)_{\subT}$, define the controlled  diffusion process by $X_0^u = \xi$ and 
\begin{equation}\label{eq:CSDE}
\ud X^u_t=\left[ b (X^u_t, t )+ u_t \right] \ud t + \sigma \ud W_t. 
\end{equation} 
We say a control $u$ is admissible if (i) $u_t$ is measurable with respect to $\sigma( (X^u_s)_{0 \leq s \leq t} )$,  (ii) the SDE~\eqref{eq:CSDE} admits a weak solution, and (iii)  $\E \int_0^T \|u_t\|^2 \ud t < \infty$, where $\| \cdot \|$ denotes the $L^2$ norm. Denote the set of all admissible controls by $\cU$.  
Note that the initial distributions of both $X$ and $X^u$ are always fixed to be $\mu_0$. 
For ease of presentation, throughout the paper we adopt the following regularity assumption on $b$, which was also used by~\citet{jamison1975markov,dai1991stochastic}:   
\begin{assumption}\label{ass.b}
For each $1 \leq i \leq d$, $b_i$ is bounded and continuous in $\R^d \times [0, T]$ and is H\"{o}lder continuous in $x$, uniformly with respect to $(x, t) \in \R^d \times [0, T]$. 
\end{assumption}
Under Assumption~\ref{ass.b}, $X$ has a transition density function $p(x, t\mid y, s)$; that is, for any $0 \leq s < t \leq T$,  $y \in \R^d$, and Borel set $A$ in $\R^d$, 
\begin{equation}\label{eq:def.p}
    \E[  X_t  \in A \mid X_s = y ]  = \int_A  p(x, t \mid y, s)  \lambda (\ud x).
\end{equation}
We will use $\ud x$ as a shorthand for $\lambda (\ud x)$. 
Moreover, by Girsanov theorem, Assumption~\ref{ass.b} implies that the probability measures induced by $(X_t)_{\subT}$ and $(B_t)_{\subT}$ are equivalent, where $B_t = \xi + \sigma W_t$.  
Hence, for any $t > s \geq 0$, $p(x, t \mid y, s)$ is strictly positive and $\Law(X_t) \approx \lambda$ (i.e., two measures are equivalent).   
The role of Assumption~\ref{ass.b} in our theoretical results will be further discussed in Remark~\ref{rmk:regularity}. 
 
Schr\"{o}dinger bridge aims to find a minimum-energy modification of the dynamics of $X$ so that its terminal distribution coincides with a pre-specified distribution $\mu_T$, where ``energy'' is measured by KL divergence. 
Given $\sigma$-finite measures $\nu$ and $\mu$ such that $\nu \ll \mu$, we use $\KL(\nu, \mu) = \int \log (  \frac{\ud \nu }{\ud \mu} ) \ud \nu$ to denote the KL divergence; if $\nu \not\ll \mu$, define $\KL(\nu, \mu) = \infty$. 
\citet{dai1991stochastic}  considered the following stochastic control formulation of Schr\"{o}dinger bridge. 

\begin{problem}[Schr\"{o}dinger bridge]\label{problem0}  
Let $\cU_0 = \{ u \in \cU  \colon \Law(X^u_T) = \mu_T \}$ where $(X^u_t)_{\subT}$ is defined in~\eqref{eq:CSDE}. Find $V = \inf_{ u \in \cU_0 } J(u)$, where 
\begin{equation}\label{eq:value.SB0}
     J(u) =  \E \int_0^T \frac{\|u_t\|^2}{2\sigma^2} \ud t, 
\end{equation}
and find the optimal control $u^*$ such that $J(u^*) = V$. 
\end{problem}

\begin{remark}\label{rmk:KL}
Let $\P_X$ (resp. $\P_X^u$) denote the probability measure induced by $X$ (resp. $X^u$) on the space of continuous functions on $[0, T]$. 
For any admissible control $u \in \cU$, Girsanov theorem implies that $J(u) = \KL(\P_X^u, \P_X)$. 
\end{remark}

Problem~\ref{problem0} has been well studied in the literature. 
In the special case where $\mu_0$ is a Dirac measure, the solution can be succinctly described, and $V$ is just the KL divergence between two probability distributions; we recall this in Theorem~\ref{th:dp1} below.  
In this paper, $\nabla$ always denotes differentiation with respect to $x$. 

\begin{theorem}[\citet{dai1991stochastic}]\label{th:dp1}
Let $\mu_0$ be the Dirac measure such that $\mu_0(\{x_0\}) = 1$ for some $x_0 \in \R^d$, and let 
$X$ be a weak solution to~\eqref{eq:udp}. 
Assume  $\KL(\mu_T, \Law(X_T)) < \infty$.  
For Problem~\ref{problem0}, the optimal control is given by $u^*_t = \sigma^2 \nabla \log h(X_t^{u^*}, t)$, where 
\begin{equation}
\begin{aligned}
    h(x, t) &= \int p(z, T \mid x, t)   \frac{ f_T(z)}{ p(z, T \mid x_0, 0)}    \ud z \eqqcolon \E\left[   \frac{ f_T(X_T)}{ p(X_T, T \mid x_0, 0)}  \,\Big|\, X_t = x\right].  
\end{aligned}
\end{equation}
Moreover, $J(u^*) = \KL(\mu_T, \Law(X_T))$. 
\end{theorem}

\begin{remark}  
If $\KL(\mu_T, \Law(X_T)) = \infty$, then Problem~\ref{problem0} does not admit a solution in the sense that no admissible control $u$ can yield $\Law(X_T^u) = \mu_T$. 
\end{remark}

We propose a relaxed stochastic control formulation of the Schr\"{o}dinger bridge problem by allowing the distribution of $X^u_T$ to be different from $\mu_T$. 
\begin{problem}[Soft-constrained Schr\"{o}dinger bridge]\label{problem1}  
For $\beta > 0$, find $V = \inf_{ u \in \cU  } J_\beta(u)$, where 
\begin{equation}\label{eq:value.SB1}
\begin{aligned}
     & J_\beta(u) 
     =\,  \beta \, \KL (\Law(X^u_T), \mu_T) +  \E \int_0^T \frac{\|u_t\|^2}{2\sigma^2} \ud t, 
\end{aligned}
\end{equation} 
and find the  optimal control $u^*$ such that $J_\beta(u^*) = V$. 
\end{problem}

Problem~\ref{problem1} (i.e., the SSB problem) replaces the hard constraint $\Law(X^u_T) = \mu_T$ in Problem~\ref{problem0} with a soft constraint parameterized by $\beta$.
When $\beta = 0$, it is clear that the optimal control $u^*$  for Problem~\ref{problem1} is $u^*_t \equiv 0$. 
As $\beta \rightarrow \infty$, the law of $X_T^u$ is forced to agree with $\mu_T$, and we will see in Theorem~\ref{th:sol.dirac} that the optimal control for Problem~\ref{problem1} converges to that for Problem~\ref{problem0}. 

Before we try to solve Problem~\ref{problem1} in full generality, we make a remark on how Problem~\ref{problem0} can be simplified.  
In the literature, Problem~\ref{problem0} is often called the dynamic Schr\"{o}dinger bridge problem. Since the objective function~\eqref{eq:value.SB0} is the KL divergence between the laws of the controlled and uncontrolled processes (recall Remark~\ref{rmk:KL}), we can use an additive property of KL divergence to reduce Problem~\ref{problem0} to a static version~\citep{leonard2013survey}, where one only needs to find a joint distribution $\pi$ with marginals $\mu_0$ and $\mu_T$ that minimizes $\KL(\pi,  \Law(X_0, X_T))$. 
Although this property is not directly used in this paper, the insight from this observation underpins our stochastic control analysis of SSB.  
In particular, when $\mu_0$ is a Dirac measure,  the solution to Problem~\ref{problem1} can be obtained by a simple argument  which reduces the problem to optimizing over the distribution of $X_T^u$ instead of over the distribution of the whole process $(X_t^u)_{\subT}$.

\begin{theorem}\label{th:sol.dirac}
Let $\mu_0$ be as given in Theorem~\ref{th:dp1}.  
For Problem~\ref{problem1}, the optimal control is given by $u^*_t = \sigma^2 \nabla \log h(X_t^{u^*}, t)$, where 
\begin{align*}
    \;& h(x, t) =   h(x, t; \beta) 
    =\; C^{-1} \int p(z, T \mid x, t) \cdot \left( \frac{ f_T(z)}{p(z, T \mid x_0, 0) } \right)^{\beta / (1 + \beta)} \ud z,
\end{align*}
and $C = \int f_T(x)^{\beta/(1 + \beta)} p(x, T \mid x_0, 0)^{1 / (1 + \beta)} \ud x$.   
Moreover, $J_\beta(u^*) = - (1 + \beta) \log C \in [0, \infty),$  and 
\begin{equation}\label{eq:limit.h.dirac}
\lim_{\beta \rightarrow \infty} h(x, t; \beta) =  \int p(z, T \mid x, t)  \frac{ f_T(z)}{p(z, T \mid x_0, 0) }  \ud z. 
\end{equation}
\end{theorem} 
\begin{proof} 
If $\Law(X^u_T) \not\ll \mu_T$, then $u$ cannot be optimal since $J_\beta(u) = \infty$.  Now fix an arbitrary $u \in \cU$ such that $\Law(X^u_T) = \mu \ll \mu_T$. 
Letting $J(u)$ be as given in~\eqref{eq:value.SB0}, we have  
\begin{align*}
J_\beta(u) = \;&   \beta \, \KL (\mu, \mu_T) + J(u) \\
\geq \;& \beta \, \KL (\mu, \mu_T) + \KL( \mu, \Law(X_T)) \\
\eqqcolon \;& \tilde{J}_\beta(\mu)
\end{align*}
where the inequality follows from Theorem~\ref{th:dp1}. 
Since $\mu_T$ has density $f_T$ and $ \Law(X_T)$ has density $ p(x, T \mid x_0, 0)$, 
we can apply Lemma~\ref{lm:min.KL} in \Appx{} to get 
     $\inf_{\mu}  \tilde{J}_\beta(\mu) =  - (1 + \beta) \log C \in [0, \infty),$
where the infimum is taken over all probability measures on $\R^d$ and is attained at $\mu^*$ such that 
\begin{align*}
\begin{aligned}
        \frac{\ud \mu^*}{\ud \lambda}(x) = C^{-1} f_T(x)^{\beta / (1 + \beta)} p(x, T \mid x_0, 0)^{1 / (1 + \beta)}. 
\end{aligned}
\end{align*}
The convergence of $h(x, t; \beta)$ as $\beta \rightarrow \infty$ also follows from Lemma~\ref{lm:min.KL}. 

It only remains to prove that $J_\beta(u^*) = - (1 + \beta) \log C$. Observe that we can rewrite $h$ as 
\begin{align*}
     h(x, t) =\;&   \int p(z, T \mid x, t)   \frac{  ( \ud \mu^* / \ud \lambda )(z) }{  p(z, T \mid x_0, 0)  }  \ud z. 
\end{align*}
Hence, Theorem~\ref{th:dp1} implies that $u^*$ is also the solution to Problem~\ref{problem0} where the terminal constraint is given by $\mu^*$. So Theorem~\ref{th:dp1} yields that  $J_\beta(u^*) = \tilde{J}_\beta(\mu^*)$, which proves the claim.   
\end{proof}

\begin{remark}\label{rmk:simulate}
When $b$ is constant, the transition density $p(x, t \mid x_0, 0)$ is easy to evaluate.
If $f_T$ is known up to a normalizing constant, one can then use the Monte Carlo sampling scheme proposed by~\citep{huang2021schrodinger} to approximate the drift $b + \sigma^2 \nabla \log h(x, t)$ and simulate the controlled diffusion process~\eqref{eq:CSDE}.  
We describe this method and generalize it using importance sampling techniques in \Appx{}~\ref{sec:toy.examples}. 
More sophisticated score-based sampling schemes can also be applied~\citep{heng2024diffusion}. 
\end{remark}

One difference between Theorems~\ref{th:dp1} and~\ref{th:sol.dirac} is that the condition $\KL(\mu_T,  \Law(X_T)) < \infty$ is not required for solving Problem~\ref{problem1}. 
We give a toy example illustrating the importance of this difference.  

\begin{example}\label{ex:cauchy}
Consider $b \equiv 0$,  $T = 1$, $x_0 = 0$ and $\mu_T$ being the Cauchy distribution. Then, $\Law(X_T)$ is just the normal distribution with mean zero and covariance $\sigma^2 I$, and we have $\KL(\mu_T, \Law(X_T)) = \infty$. 
Problem~\ref{problem0} does not admit a solution in this case, but Problem~\ref{problem1}  has a solution for any $\beta \in [0, \infty)$ and the associated optimal control has finite energy cost. 
In \Appx{}~\ref{sec:cauchy}, we simulate the solution to Problem~\ref{problem1} with $\mu_T$ being the Cauchy distribution. 
We find that when $\beta = \infty$, the numerical scheme  is unstable and fails to capture the heavy tails of the Cauchy distribution. In contrast,  using a finite value of $\beta$ significantly stabilizes the algorithm. 
\end{example}

\section{Solution to Soft-constrained Schr\"{o}dinger Bridge}\label{sec:sol.RSB}
When $\mu_0$ is not a Dirac measure, the solution to the Schr\"{o}dinger bridge problem is more difficult to describe and is  characterized by  the so-called Schr\"{o}dinger system~\citep[Theorem 2.8]{leonard2013survey}. 
In this section, we prove that the solution to Problem~\ref{problem1} can be obtained in a similar way, but the Schr\"{o}dinger system for Problem~\ref{problem1} now depends on $\beta$. 

The main idea behind our approach is to first show that the optimal control must belong to a small class parameterized by a function $g$ and then use an argument similar to the proof of Theorem~\ref{th:sol.dirac} to determine the choice of $g$. 
To introduce this class of controls, for each  measurable function $g \colon \R^d \rightarrow [0, \infty)$, let $\cT g$ denote the function on  $\R^d \times [0, T)$ given by 
\begin{equation}\label{eq:def.Tg}
    (\cT g)(x, t) = \sigma^2 \nabla \log h(x, t), 
\end{equation}
where 
\begin{equation}\label{eq:def.h}
    h(x, t)= \E[ g(X_T) \mid X_t = x ].
\end{equation}
Let $\cU_1 = \{ u \in \cU \colon u_t = (\cT g)(X^u_t, t) \text{ for some } g \geq 0 \}$ denote the set of all controls that are constructed by this logarithmic transformation. 
We present in Theorem~\ref{th:h} in \Appx{} some well-known results about the controlled SDE~\eqref{eq:CSDE} with  $u \in \cU_1$; in particular, part (iv) of  the theorem shows that such a process is a Doob's $h$-path process~\citep{doob1959markov}. 
Theorem~\ref{th:h} is largely adapted from Theorem 2.1 of~\citet{dai1991stochastic}, and similar results are extensively documented in the literature~\citep{jamison1975markov,  fleming1985stochastic, follmer1988random,  MR731258,  fleming2005logarithmic}. 
We can now prove a key lemma.

\begin{lemma}\label{lm:girsanov} 
Let $u$ be an admissible control. 
Let $h$ be as given in~\eqref{eq:def.h} for some measurable $g \geq 0$ such that $\E g(X_T) < \infty$ and $h > 0$ on $\R^d \times [0, T)$.  
We have  
\begin{align*}
    J_\beta(u) \geq  \;& \beta \, \KL (\Law(X^u_T),\mu_T) + 
    \;   \E[\log g(X_T^u)] 
    - \int \log h(x, 0) \mu_0(\ud x), 
\end{align*}
where $J_\beta$ is defined in~\eqref{eq:value.SB1}. 
The equality holds when $u_t = (\cT g)(X^u_t, t)$.  
\end{lemma}
\begin{proof}
    See \Appx{}~\ref{sec:proofs}.  
\end{proof}

\begin{remark}\label{rmk:regularity}
Assumption~\ref{ass.b} guarantees that the SDE~\eqref{eq:udp} admits a unique (in law) weak solution. More importantly, in the proof of Theorem~\ref{th:h} (which is used to derive Lemma~\ref{lm:girsanov}), Assumption~\ref{ass.b} is used to ensure that  SDE~\eqref{eq:udp} has a transition density function such that the function $h$ defined in~\eqref{eq:def.h} is sufficiently smooth and satisfies $\frac{\partial h}{\partial t} +  \cL h = 0$, where $\cL$ denotes the generator of $X$ (see Theorem~\ref{th:h}). 
This condition can be relaxed; see~\citet{friedman1975stochastic} and~\citet[Chap. 5.7]{karatzas2012brownian} for details. 
\end{remark}

Observe that in the bound given in Lemma~\ref{lm:girsanov}, the term $\int \log h(x, 0) \mu_0(\ud x)$ is independent of the control $u$, and the other terms depend on $u$ only through the  distribution of $X^u_T$. 
This implies that  among all the admissible controls that result in the same distribution of $X^u_T$, the cost $J_\beta$ is minimized by  some $u \in \cU_1$.  
We can now prove the main theoretical result of this work in Theorem~\ref{th:sol.RSB}. The existence of the solution will be considered later in Theorem~\ref{th:exist.compact}. 
  
\begin{theorem}\label{th:sol.RSB}
Suppose there exist $\sigma$-finite measures $\nu_0, \nu_T$ such that $\nu_0 \approx \mu_0, \nu_T \approx \mu_T$ and
\begin{align}
   \frac{\ud \mu_0}{\ud \nu_0}(y) =\;& \int p(x, T \mid y, 0) \nu_T(\ud x), \label{eq:rsb2} \\ 
   \frac{\ud \mu_T}{\ud \lambda}(x) =\;& \rho_T(x)^{\frac{1+\beta}{\beta}} \int p(x, T \mid y, 0) \nu_0(\ud y), \label{eq:rsb1} 
\end{align}
where $\rho_T = \ud \nu_T / \ud \lambda$ and the transition density $p$ is defined in~\eqref{eq:def.p}. 
Assume $\int  ( \ud \mu_0 / \ud \nu_0) \ud \mu_0  < \infty$. 
Then   $u^*_t = \sigma^2 \nabla \log h(X_t^{u^*}, t)$ solves Problem~\ref{problem1}, where 
$h(x, t) =  \E[ \rho_T(X_T) \mid X_t = x].$ 
Moreover, $J_\beta(u^*) = -\KL(\mu_0, \nu_0) \in [0, \infty)$. 
\end{theorem} 

\begin{proof}
    See \Appx{}~\ref{sec:proofs}. 
\end{proof}

\begin{remark}\label{rmk:geom.mix}
    As we derive in the proof, the terminal distribution of the optimally controlled process is still a geometric mixture of two distributions. Explicitly, its density is proportional to 
    $$    f_T(x)^{\beta / (1 + \beta)} \left(\int p(x, T \mid y, 0) \nu_0 (\ud y) \right)^{1 / (1 + \beta)},$$   
    where we recall $f_T$ is the density of $\mu_T$.  
\end{remark}

\begin{remark}\label{rmk:RSB.unique}
The assumption $\int  ( \ud \mu_0 / \ud \nu_0) \ud \mu_0  < \infty$ guarantees that we can use Lemma~\ref{lm:girsanov} and Theorem~\ref{th:h} and that $J_\beta(u^*) < \infty$; it is also used in~\citet{dai1991stochastic}.
Observe that $u^*_t = \sigma^2 \nabla \log h(X_t^{u^*}, t)$ is invariant to the scaling of the function $\rho_T$ and thus also the scaling of $\nu_0$ and $\nu_T$.  
This suggests that the system defined by~\eqref{eq:rsb2} and~\eqref{eq:rsb1} can be generalized as follows. Let $a > 0$ and $\sigma$-finite measures $\nu_0 = \nu_0(a), \nu_T = \nu_T(a)$ be the solution to the following system
\begin{align}
   \frac{\ud \mu_0}{\ud \nu_0}(y) =\;& \int p(x, T \mid y, 0) \nu_T(\ud x),  \label{eq:rsb5} \\ 
   \frac{\ud \mu_T}{\ud \lambda}(x) =\;& (a \rho_T(x) )^{\frac{1 + \beta}{\beta}}  \int p(x, T \mid y, 0) \nu_0(\ud y), \quad \label{eq:rsb6}
\end{align} 
where $\rho_T = \ud \nu_T / \ud \lambda$. We can use essentially the same argument to show that the choice $h(x, t) = \E[ \rho_T(X_T) \mid X_t = x]$ is optimal, but now we have 
$$J_\beta(u^*) = -(1 + \beta) \log a - \KL(\mu_0, \nu_0(a)).$$ 
This is the same as that given in Theorem~\ref{th:sol.RSB}. Indeed,  if $(\nu_0^*, \nu_T^*)$ is a solution to~\eqref{eq:rsb2} and~\eqref{eq:rsb1}, then the solution to~\eqref{eq:rsb5} and~\eqref{eq:rsb6} is given by $\ud \nu_0(a) = a^{1 + \beta} \ud \nu_0^*$ and $\ud \nu_T(a) = a^{-(1 + \beta)}\ud \nu_T^*$.
\end{remark}

\begin{example}\label{ex:dirac} 
Theorem~\ref{th:sol.dirac} can be obtained from Theorem~\ref{th:sol.RSB}  as a special case. If $\mu_0$ is a Dirac measure with $\mu_0(\{x_0\}) = 1$, one can check that the solution to~\eqref{eq:rsb2} and~\eqref{eq:rsb1} is given by  $\nu_0 = \mu_0$ and 
$$ \rho_T(x) = \left( \frac{f_T(x)}{ p(x, T\mid x_0, 0)} \right)^{\beta/(1 + \beta)}.$$ 
\end{example}

\begin{example}\label{ex:rev}
Suppose $b \equiv 0$, and let $\phi_\sigma$ denote the density function of the normal distribution with mean $0$ and covariance matrix $\sigma^2 I$. 
We have $p(x, T \mid y, 0) = \phi_{ \sigma \sqrt{T} }(x - y)$. 
Assume $\mu_0 \ll \lambda$ has density $f_0$, and suppose that $f_0, f_T$ satisfy 
\begin{align*}
    f_0(y) = c^{-1} \int \phi_{ \sigma \sqrt{T} }(x - y) f_T(x)^{\frac{\beta}{ 1 + \beta}} \ud x, 
\end{align*}
where $c = \int f_T(x)^{\beta/(1+\beta)} \ud x$ is the normalizing constant assumed to be finite. A routine calculation using $\int \phi_{ \sigma \sqrt{T} }(x - y) \ud y = 1$ can verify that the solution to~\eqref{eq:rsb2} and~\eqref{eq:rsb1} is given by 
\begin{align*}
    \frac{\ud \nu_0}{ \ud \lambda} = c^{-(1 + \beta)}, \; \rho_T(x) = c^\beta f_T(x)^{\beta/(1 + \beta)}. 
\end{align*}
According to Remark~\ref{rmk:RSB.unique}, by choosing $a = c$ in~\eqref{eq:rsb6}, we can also replace $\nu_0, \nu_T$ by $\nu_0(a), \nu_T(a)$, where $\nu_0(a)$ coincides with $\lambda$ and $\nu_T(a)$ is a probability distribution with density $c^{-1} f_T^{\beta/(1+\beta)}$. 
For the original Schr\"{o}dinger bridge problem (i.e., $\beta = \infty$), this solution has been used in developing efficient generative sampling methods~\citep{pmlr-v139-wang21l, berner2022optimal}.
\end{example}

\cite{chen2019relaxed} studied a matrix optimal transport problem which can be seen as the discrete analogue to Problem~\ref{problem1}, and they proved the solution to the corresponding Schr\"{o}dinger system admits a unique solution. 
The main idea is to show that the solution can be characterized as the fixed point of some operator with respect to the Hilbert metric~\citep{bushell1973hilbert}, a technique that has been widely used in the literature on Schr\"{o}dinger system~\citep{fortet1940resolution, georgiou2015positive, chen2016entropic, essid2019traversing, deligiannidis2024quantitative}. 
For the Schr\"{o}dinger system defined by~\eqref{eq:rsb2} and~\eqref{eq:rsb1}, an argument based on the Hilbert metric 
can also be applied to prove the existence and uniqueness of the solution when $\mu_0, \mu_T$ are absolutely continuous and have compact support. 

\begin{theorem}  \label{th:exist.compact}
Let $K_0$ (resp. $K_T$) denote the support of $\mu_0$ (resp. $\mu_T$). Assume that 
\begin{enumerate}[(i)]
    \item $K_0, K_T \subset \mathbb{R}^d$ are compact; 
    \item $f_0 = \ud\mu_0 / \ud \lambda$, $f_T = \ud \mu_T / \ud \lambda$ exist, where $\lambda$ is the Lebesgue measure; 
    \item  $(y, x) \mapsto p(x, T \mid y, 0)$ is continuous and strictly positive on $K_0 \times K_T$. 
\end{enumerate} 
For any $\beta \in (0, \infty)$,  there exists a unique pair of non-negative, integrable functions $(\rho_0, \rho_T)$ such that  
\begin{align}
    f_0(y) &= \rho_0(y) \int_{K_T} p(x, T \mid y, 0) \rho_T(x) \ud x,   \\
    f_T(x) &= \rho_T(x)^{(1 +\beta) / \beta} \int_{K_0} 
    p(x, T \mid y, 0) \rho_0(y) \ud y.  
\end{align} 
\end{theorem} 

\begin{proof}
    See \Appx{}~\ref{sec:existence}. 
\end{proof}

\begin{remark}\label{rmk:proof.exist} 
The proof of Theorem~\ref{th:exist.compact} is adapted from that for Proposition 1 in~\citet{chen2016entropic}. 
Observe that the Schr\"{o}dinger system naturally yields an iterative algorithm  for computing $\rho_0, \rho_T$. 
Given an estimate for $\rho_0$,  denoted by $\hat{\rho}_0$, we can estimate  $\rho_T$ by 
\begin{align*}
    \hat{\rho}_T(x) = \left( \frac{f_T(x)}{ \int_{K_0} 
    p(x, T \mid y, 0) \hat{\rho}_0(y) \ud y } \right)^{\beta/(1 + \beta)}, 
\end{align*}
which then can be used to update $\hat{\rho}_0$ by 
\begin{align*}
   \hat{\rho}_0^{\rm{new}}(y) = \frac{f_0(y)}{ \int_{K_T} p(x, T \mid y, 0) \hat{\rho}_T(x) \ud x }. 
\end{align*}
\citet{chen2016entropic}  considered the original Schr\"{o}dinger bridge problem (i.e., $\beta = \infty$)  and showed that this updating scheme yields a strict contraction with respect to the Hilbert metric. 
We note that when $\beta < \infty$, this argument can be potentially made easier, since the mapping $\psi \mapsto \psi^{\beta/(1+\beta)}$ for a suitable function $\psi$ can be easily shown to be a strict contraction, and thus one only needs to verify the other steps in the updating are  contractions (not necessarily strict); see Lemma~\ref{lm:O} in \Appx{}.
The full scope of consequences of this observation and the existence proof for the general case are left to future study. 
\end{remark}

\section{Extension to Time Series}\label{sec:tSB} 
Recently a time series version of Problem~\ref{problem0} was studied in~\citet{tsb}, where the goal is to generate time series samples from a joint probability distribution on $\R^{d} \times \cdots \times \R^d$.  
We can generalize our Problem~\ref{problem1} to  time series data analogously. 

\begin{problem}\label{problem2}
Consider $N$ fixed time points $0 < t_1 < \cdots < t_N = T$. 
Let $\mu_N$ be a probability distribution on $\R^{d \times N}$ such that $\mu_N \ll \lambda$.  
For $\beta > 0$, find $V = \inf_{ u \in \cU  } J^N_\beta(u)$, where 
\begin{equation}\label{eq:value.SB2}
\begin{aligned}
      J^N_\beta(u) = \;& \beta \, \KL (\Law( (X_{t_i})_{1 \leq i \leq N} ), \mu_N)  
        +  \E \int_0^T \frac{\|u_t\|^2}{2\sigma^2} \ud t, 
\end{aligned}
\end{equation} 
and find the  optimal control $u^*$ such that $J^N_\beta(u^*) = V$. 
\end{problem}
 
Recall that in Section~\ref{sec:sol.RSB}, we started by considering functions $h$ such that $h(X_t, t) = \E[g(X_T) \mid \cF_t]$ for some function $g$, where $\cF_t = \sigma( (X_s)_{0 \leq s \leq t} )$. 
It turns out that this technique can be used to solve Problem~\ref{problem2} as well, but now we need to consider conditional expectations of the form $ \E[ g(X_{t_1}, \dots, X_{t_N}) \mid \cF_t ]$.  
To simplify the notation, we will write $\bx_j = (x_1, \dots, x_j)$, $\bX_j = (X_{t_1}, \dots, X_{t_j})$, and $\bX^u_j = (X_{t_1}^u, \dots, X_{t_j}^u)$; when $j = 0$, $\bx_j$, $\bX_j$, $\bX^u_j$ all denote the empty vector.  

Given a measurable function $g \colon \R^{d \times N} \rightarrow [0, \infty)$, 
the Markovian property of $X$ enables us to express the conditional expectation $g(\bX_N)$ by 
\begin{equation}\label{eq:hj1}
    \E[ g(\bX_N) \mid \cF_t ] = \sum_{j = 0}^{N-1} \ind_{  [ t_j, t_{j+1}) }(t) \cdot  h_j(X_t, t;  \bX_j),
\end{equation}
where we set $t_0 = 0$, and the function $h_j$ is defined by   
\begin{equation}\label{eq:def.hj}
\begin{aligned}
     & h_j(x, t; \bx_j) 
    =\; \E\left[g(\bx_{j}, X_{t_{j+1}}, \dots, X_{t_N} ) \,\Big|\, X_t = x \right],  
\end{aligned}
\end{equation} 
for $(x, t) \in \R^d \times [t_j, t_{j + 1}).$
Let $\cT_N g$ denote the function on  $\R^d \times [0, T) \times \R^{d \times N}$ given by 
\begin{equation}\label{eq:def.ut.ts}
\begin{aligned}
    & (\cT_N g)(x, t, \bx_N) 
    =\;   \sum_{j = 0}^{N-1} \ind_{  [ t_j, t_{j+1}) }(t) \cdot \sigma^2 \nabla \log h_j(x, t; \bx_j).  
\end{aligned}
\end{equation} 
We prove in Lemma~\ref{lm:girsanov.ts} in \Appx{} that it suffices to consider controls in the set 
$$\cU_N = \{ u \in \cU \colon u_t = (\cT_N g)(X^u_t, t, \bX^u_N) \text{ for some } g \geq 0 \}.$$ 
This result is  a generalization of Lemma~\ref{lm:girsanov} and obtained by applying Theorem~\ref{th:h} to each time interval $[t_j, t_{j+1})$ separately. 
Note although we express $u \in \cU_N$ as a function of $\bX^u_N$, by~\eqref{eq:def.ut.ts}, $u_t$ is measurable with respect to $\sigma((X^u_s)_{0\leq s \leq t})$. 
To formulate the Schr\"{o}dinger system for the time series SSB problem, 
let $p_N(\bx_N \mid y)$ denote the transition density from $X_0 = y$ to $\bX_N = \bx_N$, which is given by 
\begin{equation}\label{eq:def.pN}
    p_N(\bx_N \mid y) = p(x_1, t_1 \mid y, 0) \prod_{j = 1}^{N-1} p(x_{j+1}, t_{j+1}  \mid x_j, t_j). 
\end{equation}
 
\begin{theorem}\label{th:sol.ts}  
Consider Problem~\ref{problem2}. 
Suppose there exist $\sigma$-finite measures $\nu_0$ on $\R^d$ and $\nu_N$ on $\R^{d \times N}$ such that $\nu_0 \approx \mu_0, \nu_N \approx \mu_N$ and
\begin{align}
   \frac{\ud \mu_0}{\ud \nu_0}(y) =\;& \int p_N(\bx_N \mid y) \nu_N(\ud \bx_N),   \\ 
   \frac{\ud \mu_N}{\ud \lambda}(\bx_N) =\;& \rho_N(\bx_N)^{\frac{ 1+\beta }{\beta}} \int p_N(\bx_N \mid y) \nu_0(\ud y), 
\end{align}
where $\rho_N = \ud \nu_N / \ud \lambda$. 
Assume  $\int  ( \ud \mu_0 / \ud \nu_0 ) \ud \mu_0  < \infty$. 
Then  $u_t^* = (\cT_N \rho_N)(X^{u^*}_t, t, \bX^{u^*}_N)$ solves Problem~\ref{problem2}, where $\cT_N$ is defined by~\eqref{eq:def.ut.ts}. 
Moreover, $J_\beta^N(u^*) = -\KL(\mu_0, \nu_0) \in [0, \infty)$. 
\end{theorem}

\begin{proof}
    See \Appx{}~\ref{sec:proofs}. 
\end{proof}

Comparing the Schr\"{o}dinger system in Theorem~\ref{th:sol.RSB} and that in Theorem~\ref{th:sol.ts}, we see that the solution to Problem~\ref{problem2} has essentially the same structure as that to Problem~\ref{problem1}. 
The only difference is that in Theorem~\ref{th:sol.RSB} the Schr\"{o}dinger system is constructed by using the joint distribution of $(X_0, X_T)$, while in Theorem~\ref{th:sol.ts} it is replaced by the joint distribution of $(X_0, \bX_N)$.  
We also note that~\citet{tsb} only considered the time series Schr\"{o}dinger bridge problem with $\mu_0$ being a Dirac measure, and by letting $\beta \rightarrow \infty$, Theorem~\ref{th:sol.ts} gives the solution to their problem in the general case.

\section{Experiments}\label{sec:sim}

\subsection{Problem Setup}
We consider an application of SSB to robust generative modeling in the following scenario. 
Let $\DR$ denote a large collection of high-quality samples with distribution $\muR$, and let $\DT$ be a small set of noisy samples with distribution $\muT$. Our objective is to generate realistic samples resembling those in $\DT$, but due to the limited availability of training samples,  we want to leverage information from $\DR$ to enhance the sample quality. 

A natural idea is to use SSB as a regularization method to mitigate overfitting to the noisy samples in $\DT$. This can be implemented in two steps. 
For simplicity, we assume in this section that the uncontrolled process $X$ is given by $X_t = \sigma W_t$; that is, we assume $X_0 = 0$  and $b \equiv 0$ in~\eqref{eq:udp}. Then $X_T$ has density  $\phi_{\sigma \sqrt{T}}$ (recall this is the density of the normal distribution with mean $0$ and covariance matrix $(\sigma^2 T) I$).   
Let $f_{\REF}$ denote the density of $\muR$ with respect to the Lebesgue measure. 
Let $X^{\REF} = (X^{\REF}_t)_{\subT}$ be the Schr\"{o}dinger bridge targeting $\muR$ evolving by 
\begin{equation}\label{eq:def.Xref}
    \ud X_t^{\REF} = b^{\REF}(X_t^{\REF}, t) \ud t + \sigma \ud W_t,  
\end{equation}
for $t \in [0, T]$, where 
\begin{align*}
   b^{\REF}(X_t, t) =  \sigma^2 \nabla \log  \E\left[  \frac{ f_{\REF}(X_T) }{ \phi_{\sigma \sqrt{T}}(X_T) } \,\Big|\, X_t = x \right]. 
\end{align*}
Theorem~\ref{th:dp1} implies that $X^{\REF}_T$ has distribution $\muR$. 
Next, we solve Problem~\ref{problem1} using $X^{\REF}$ as the reference process and $\muT$ as the target distribution. This yields the process $X^{\TGT}$ with dynamics given by 
\begin{equation}\label{eq:def.Xtgt}
\ud X^{\TGT}_t= b^{\TGT}(X_t^{\TGT}, t) \ud t + \sigma \ud W_t,  
\end{equation}
for $t \in [0, T]$, where 
\begin{equation}
\begin{aligned}
    & b^{\TGT}(X_t^{\TGT}, t)  =   b^{\REF}(X_t^{\TGT}, t) + 
     \sigma^2 \nabla \log \E\left[  \left( \frac{ f_{\TGT}(X_T^{\REF}) }{ f_{\REF}(X_T^{\REF}) } \right)^{\beta/(1+\beta)} \,\Big|\, X_t^{\REF} = x \right].
\end{aligned}
\end{equation} 
By Theorem~\ref{th:sol.dirac}, the distribution of $X_T^{\TGT}$ will be close to $\muT$ if $\beta$ is relatively large. 
It turns out that there is no need to train $X^{\REF}$ and $X^{\TGT}$ separately. The following lemma shows that we can directly train a Schr\"{o}dinger bridge targeting a geometric mixture of $\muR$ and $\muT$. 

\begin{lemma}\label{lm:two.SB}
    Let $X_t = \sigma W_t$ and $X^{\REF}$ and $X^{\TGT}$ be as given in~\eqref{eq:def.Xref} and~\eqref{eq:def.Xtgt}. Equivalently, we can express the drift of $X^{\TGT}$ by 
    \begin{equation}\label{eq:drift.two.SB}
    \begin{aligned}
          & b^{\TGT}(X_t^{\TGT}, t)  
        =\;  \sigma^2 \nabla \log \E\left[    \frac{  f_{\REF}(X_T)^{ \frac{1}{1+\beta}}  f_{\TGT}(X_T)^{\frac{\beta}{1+\beta}} }{ \phi_{\sigma \sqrt{T}}(X_T) } \,\Big|\, X_t  = x \right].
    \end{aligned}
    \end{equation}
\end{lemma}
\begin{proof}
    See \Appx{}~\ref{sec:proof.two.SB}. 
\end{proof}

\begin{remark}\label{rmk:two.SB}
    Lemma~\ref{lm:two.SB} follows from a change-of-measure argument. It still holds if $X$ is not a Brownian motion but $X$ solves the SDE~\eqref{eq:udp} with $X_0 = x_0 \in \R^d$ almost surely (see \Appx{}~\ref{sec:proof.two.SB}). 
\end{remark}

The assumption that $\mu_0$ (the initial distribution of $X$) is a Dirac measure greatly simplifies the calculations and enables us to directly target the unnormalized density function $f_{\REF}^{1/(1+\beta)}  f_{\TGT}^{\beta/(1+\beta)}$. We will propose in the next subsection a score matching algorithm for learning this geometric mixture distribution. 
For applications where a general initial distribution is desirable, one may need to build iterative  algorithms by borrowing ideas from the iterative proportional fitting procedure~\citep{debortoli2023diffusion}. 

\begin{example}\label{ex:normal.mix}
For an illustrative toy example, let $\muR$ be a mixture of four bivariate normal distributions with means $(1, 1), (1, -1), (-1, 1)$, $(-1, -1)$ respectively and the same covariance matrix $0.05^2 I$. Let $\muT$ be a mixture of two normal distributions with means $(1.2, 0.8)$, $(-1.5, -0.5)$ respectively and the same covariance matrix $0.5^2 I$. 
So $\muT$ essentially contains two components of $\muR$ but with small bias and much larger noise. 
Since the density functions are known, we can directly simulate a Schr\"{o}dinger bridge process targeting  $f_{\REF}^{1/(1+\beta)}  f_{\TGT}^{\beta/(1+\beta)}$ using the method described in Remark~\ref{rmk:simulate}. 
We provide the results in \Appx{}~\ref{sec:sim.gaussian.mix}, which show that by targeting a geometric mixture with a moderate value of $\beta$, we can  effectively compel  the terminal distribution of the controlled process to acquire a covariance structure similar to $\muR$. 
\end{example}

\subsection{A Score Matching Algorithm for Learning Geometric Mixtures}\label{sec:exp.score.alg}
Let $\mu^*$ denote a probability distribution with density function  $f^*(x) = C^{-1} f_{\REF}(x)^{1/(1+\beta)}  f_{\TGT}(x)^{\beta/(1+\beta)}$ where $C$ is the normalizing constant. 
To generate samples from $\mu^*$, one can use existing score-based diffusion model methods, but, as we will see shortly, one major challenge is how to train the score functions without samples from the distribution $\mu^*$. 

Let $\mu^*_\sigma$ be the probability distribution with density  
\begin{equation}\label{eq:def.fsigma}
    f^*_\sigma(x)=\int f^*(x) \phi_\sigma(x - y) \ud y, 
\end{equation}
which can be thought of as a smoothed version of $f^*$, and suppose that we can generate samples from the distribution $\mu^*_\sigma$. Solving the Schr\"{o}dinger bridge problem with initial distribution $\mu^*_\sigma$ and terminal distribution $\mu^*$, we obtain the controlled process $X^*$ with $\Law(X^*_0) = \mu^*_\sigma$ and dynamics given by 
\begin{equation}\label{eq:sde.sbp}
\begin{aligned}
    \ud X^*_t =\;& b^*(X_t^*, t) \ud t + \sigma \ud W_t,   \\
    \text{ where } b^*(x, t) =\;& \sigma^2 \nabla \log f^*_{\sigma \sqrt{T - t}}(x). 
\end{aligned}
\end{equation}
The process $X^*$ satisfies  $\Law(X^*_T) = \mu^*$  (note that this result is a special case of  Example~\ref{ex:rev} with $\beta = \infty$). 

We now describe how to simulate the dynamics given in~\eqref{eq:sde.sbp} and generate samples from $\mu^*_\sigma$. 
First, to learn the drift function in~\eqref{eq:sde.sbp}, we propose to combine the score matching technique  with importance sampling. 
Let $s_\theta(x, \tilde{\sigma})$ denote our approximation to $\nabla \log f^*_{\tilde{\sigma}}(x)$, 
where  $\tilde{\sigma} \in [0, \sigma \sqrt{T}]$ 
and the unknown parameter $\theta$ typically denotes a neural network.  
According to the well-known score matching technique~\citep{hyvarinen2005estimation, vincent2011connection}, we can estimate $\theta$ for a given $\tilde{\sigma}$ by minimizing the objective function  
$\E_{x \sim \mu^*} [ L (x, \theta; \tilde{\sigma})]$,  where 
\begin{align*} 
     L (x, \theta; \tilde{\sigma}) 
   =\;& \E_{\tilde{x} \sim \mathcal{N}(x, \tilde{\sigma}^2 I)} 
    \left[\left\| s_\theta( \tilde{x}, \tilde{\sigma}) - \nabla_{\tilde{x}} \log f^*_{\tilde{\sigma}}(\tilde{x} \mid x)\right\|^2  \right] \\
    =\;& \E_{\tilde{x} \sim \mathcal{N}(x, \tilde{\sigma}^2 I)} 
    \left[\left\| s_\theta( \tilde{x}, \tilde{\sigma}) + \frac{\tilde{x} - x}{\tilde{\sigma}^2 } \right\|^2  \right]. 
\end{align*}
Unfortunately, the existing score matching methods estimate $\E_{x \sim \mu^*} [ L (x, \theta; \tilde{\sigma})]$ by using samples from $\mu^*$, which are not applicable to our problem since we only have access to samples from $\muR$ and $\muT$ but not from $\mu^*$. 
We propose to use importance sampling to tackle this issue. 
Let $\tilde{\mu}$ denote an auxiliary distribution with density $q = \ud \tilde{\mu} / \ud \lambda$, and assume that we can generate samples from $\tilde{\mu}$ (e.g. we can let $\tilde{\mu} = \muR, \muT$ or the smoothed versions of them). 
By a change of measure, we can express $\E_{x \sim \mu^*} [ L (x, \theta; \tilde{\sigma})]$ as 
\begin{equation}\label{eq:is.score}
    C^{-1} \, \E_{x \sim \tilde{\mu}} \left[ L (x, \theta; \tilde{\sigma}) \left( \frac{f_{\REF}(x)}{q(x)} \right)^{\frac{1}{1 + \beta}} 
    \left( \frac{f_{\TGT}(x)}{q(x)} \right)^{\frac{\beta}{1 + \beta}}  
     \right]. 
\end{equation}
The density ratios $f_{\REF} / q$ and $f_{\TGT}/q$ can be learned by using samples from $\muR, \muT, \tilde{\mu}$ and minimizing the logistic regression loss as in~\cite{pmlr-v139-wang21l}. We do not need to know $C$ since it does not affect the drift term in~\eqref{eq:sde.sbp}. 
In particular, when $\tilde{\mu} = \muR$, we get 
\begin{align*} 
 & \E_{x \sim \mu^*} [ L (x, \theta; \tilde{\sigma})] 
=\;
    C^{-1} \, \E_{x \sim  \muR} \left[ L (x, \theta; \tilde{\sigma}) \left( \frac{f_{\TGT}(x)}{f_{\REF}(x)} \right)^{\frac{\beta}{1 + \beta}}  \right], 
\end{align*}
where the ratio $f_{\TGT} / f_{\REF}$ can be learned by using samples from $\DT$, $\DR$. 
A similar expression can be easily derived for $\tilde{\mu} = \muT$. 
This importance sampling method enables us to estimate the expectation $\E_{x \sim \mu^*} [ L (x, \theta; \tilde{\sigma})]$ by using samples from $\DT, \DR$.   
By averaging over $\tilde{\sigma}$ randomly drawn from the interval $[0, \sigma \sqrt{T}]$, we obtain an estimated loss for the parameter $\theta$. Minimizing this loss we get the estimate $\hat{\theta}$, and we can approximate $\nabla \log f^*_{\tilde{\sigma}}(x)$ using $s_{\hat{\theta}}(x, \tilde{\sigma})$. 

Simulating the SDE~\eqref{eq:sde.sbp} requires us to generate samples from $\mu^*_\sigma$. There are a few possible approaches. 
First, if $\sigma$ is chosen to be sufficiently large, one may argue that we can simply approximate $f^*_\sigma$ using the normal density $\phi_\sigma$, and thus we only need to draw $X_0$ from the normal distribution. This is  the approach taken in the score-based generative models based on backward SDEs~\citep{song2019generative}. 
Second, one can run an additional Schr\"{o}dinger bridge process with terminal distribution $\mu^*_\sigma$, as proposed in the two-stage Schr\"{o}dinger bridge algorithm of~\cite{pmlr-v139-wang21l}. The dynamics they considered is simply the solution to Problem~\ref{problem0} given in Theorem~\ref{th:dp1}, where the uncontrolled process is a Brownian motion started at $0$. The drift term is approximated by a Monte Carlo scheme (see Remark~\ref{rmk:simulate} and \Appx{}~\ref{sec:MC.simulation}), which also requires an estimate of the density ratio $f^*_\sigma / \phi_\sigma$ and the score of $ f^*_\sigma$. 
Third, one can also  run the Langevin diffusion
$\ud \tilde{X}_t =  \sigma^2 \nabla \log f^*_\sigma( \tilde{X}_t ) \ud t + \sigma \ud W_t$ for a sufficiently long duration, which has stationary distribution $\mu^*_\sigma$. We found in our experiments that this approach yields more robust results, probably because it does not require Monte Carlo sampling which may yield drift estimates with large variances. 
 
\subsection{An Example Using MNIST}\label{sec:mnist}
We present a numerical example using the MNIST data set~\citep{deng2012mnist}, which consists of images of handwritten digits from $0$ to $9$. All images have size $28\times28$ and pixels are rescaled to $[0, 1]$. 
To construct $\DT$, we randomly select 50 images labeled as eight and reduce their quality by adding Gaussian noise with mean 0 and variance $0.4$; see Fig.~\ref{fig:DT} in \Appx{}~\ref{sec:mnist.supp}. 
The reference data set $\DR$ includes all images that are not labeled as eight, a total of 54,149 samples.  

We first run the algorithm of~\cite{pmlr-v139-wang21l} using only  $\DT$, and the generated samples are noisy as expected; see Fig.~\ref{fig:path2} in \Appx{}~\ref{sec:mnist.supp}. 
Next, we run our algorithm described in Section~\ref{sec:exp.score.alg} with different choices of $\beta$. The density ratio and scores are trained by using one GPU (RTX 6000).  
Generated samples are shown in Fig.~\ref{fig:path1}, and we report in Table~\ref{table:fid} the Fr\'{e}chet Inception Distance score~\citep{heusel2017gans} assessing  the disparity between our generated images (sample size = 40K) and the collection of clean digit 8 images from the MNIST dataset (sample size $\approx$ 6K). 
When $\beta$ is too small, the generated images do not resemble those in $\DT$ and we frequently observe the influence of other digits;  if $\beta = 0$, the images are essentially generated from $\muR$.   
When $\beta$ is too large, the influence from the reference data set becomes negligible, but the algorithm tends to overfit to the noisy data in $\DT$. 
For moderate values of $\beta$, we observe a blend of characteristics from both data sets, and when $\beta=1.5$, FID is minimized (among all tried values) and we get high-quality images of digit 8. 
This experiment illustrates that the information from $\DT$ can help capture the structural features associated with the digit 8, while the samples from $\DR$ can guide the algorithm towards effective noise removal. 
In \Appx{}~\ref{sec:mnist.supp}, we further analyze the generated images using t-SNE plots and inception scores~\citep{salimans2016improved}. 

\begin{figure}[H]
    \centering
    \includegraphics[width=0.65\linewidth]{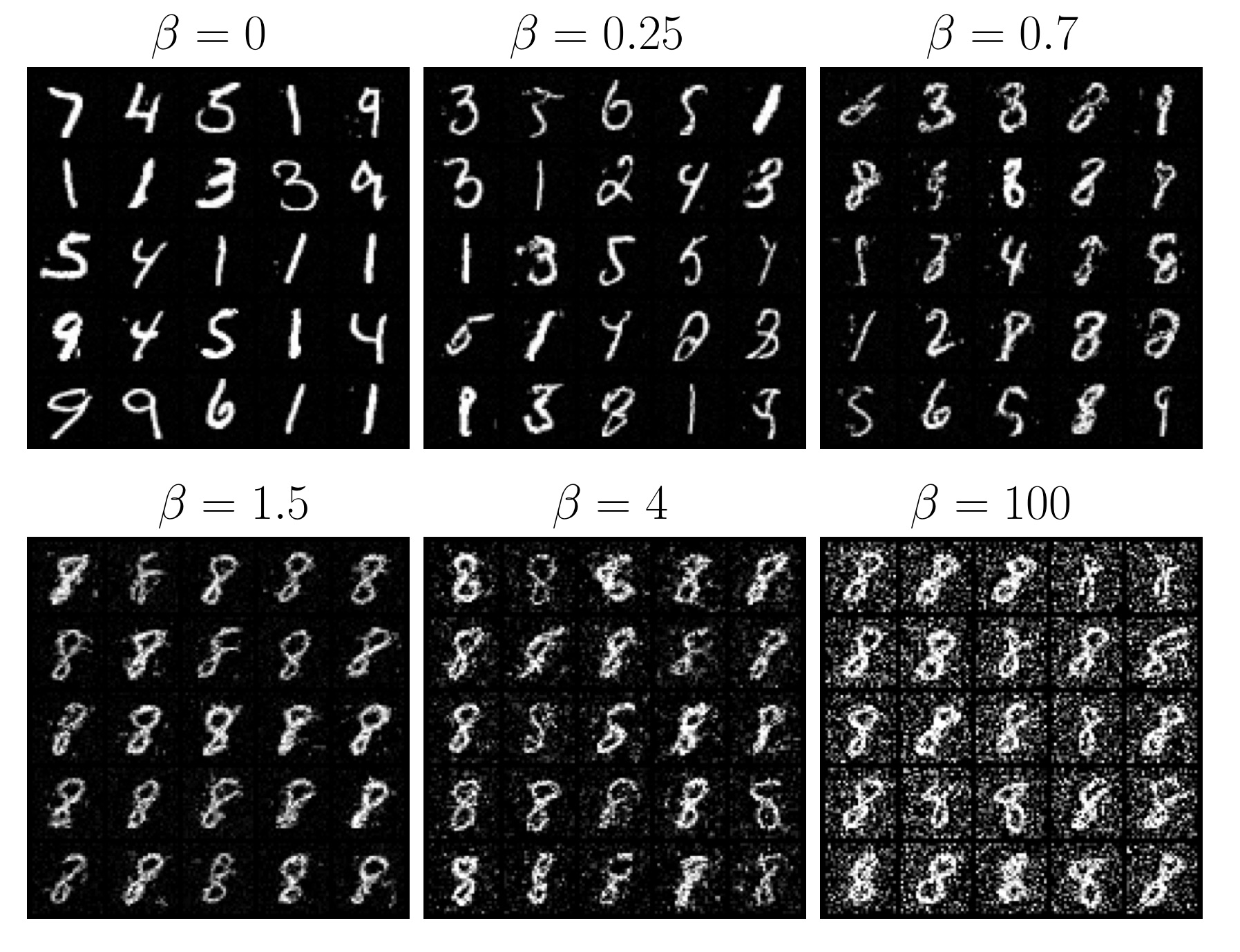}    
    \caption{SSB Samples for MNIST Experiment}\label{fig:path1}
\end{figure}

\begin{table}[H]
  \caption{FID Scores for MNIST Experiment}\label{table:fid}
\begin{center}
   \begin{tabular}{ccccccc}
      \hline
     $\beta$  & 0 & 0.25 & 0.7 & 1.5 & 4 & 100 \\
      \hline 
     FID  &  67.4 & 66.4 & 61.0 & 56.3 & 110.9 & 182.4 \\
      \hline
  \end{tabular}
\end{center}
\end{table}

\section{Concluding Remarks}\label{sec:conc}
We propose the soft-constrained Schr\"{o}dinger bridge (SSB) problem and find its solution. Our theory encompasses the existing stochastic control results for Schr\"{o}dinger bridge in~\citet{dai1991stochastic} and~\citet{tsb} as special cases. 
The paper focuses on the theory of SSB, and the numerical examples are designed to be uncomplicated but illustrative. 
More advanced algorithms for solving Problems~\ref{problem1} and~\ref{problem2} in full generality need to be developed. 
It will also be interesting to study the applications of SSB to other generative modeling tasks, such as conditional generation, style transfer~\citep{shi2022conditional, shi2023schrowave, su2022dual} and time series data generation~\citep{tsb}. 
Some further generalization of the objective function may be considered as well; for example, one can add a time-dependent cost as in~\cite{pra1990markov} or consider a more general form of the terminal cost.

\section{Acknowledgements}
The authors would like to thank Tiziano De Angelis for valuable discussion on the problem formulation, 
Yun Yang for the helpful conversation on the numerics, and anonymous reviewers for their comments and suggestions.  
JG and QZ were supported in part by NSF grants  DMS-2311307 and DMS-2245591. 
XZ acknowledges the support from NSF DMS-2113359. 
The numerical experiments were conducted with the computing resources provided by Texas A\&M High Performance Research Computing.

\bibliographystyle{plainnat}
\bibliography{Ref}

\clearpage 
\newpage 

\begin{center}
    \Large{\textbf{APPENDIX}}  
\end{center}
\medskip 

The code for all experiments (Cauchy simulation in \Appx{}~\ref{sec:cauchy}, normal mixture simulation in \Appx{}~\ref{sec:sim.gaussian.mix} and MNIST example in Section~\ref{sec:mnist}) is at the GitHub repository \href{https://github.com/gargjhanvi/Soft-constrained-Schrodinger-Bridge-a-Stochastic-Control-Approach}{https://github.com/gargjhanvi/Soft-constrained-Schrodinger-Bridge-a-Stochastic-Control\\-Approach}.  
\setcounter{section}{0} 
\renewcommand{\thesection}{\Alph{section}} 
\renewcommand{\thetheorem}{\thesection\arabic{theorem}}
\setcounter{theorem}{0}
\renewcommand{\thelemma}{\thesection\arabic{lemma}}
\renewcommand{\thecorollary}{\thesection\arabic{corollary}}
\renewcommand{\thedefinition}{\thesection\arabic{definition}}
\renewcommand{\theexample}{\thesection\arabic{example}}
\setcounter{example}{0}
\renewcommand{\thefigure}{\thesection\arabic{figure}}
\setcounter{figure}{0}
\renewcommand{\thetable}{\thesection\arabic{table}}
\setcounter{table}{0}
\renewcommand{\theremark}{\thesection\arabic{remark}}
\setcounter{remark}{0}

\section{Simulation with Known Density Functions}\label{sec:toy.examples}

\subsection{Monte Carlo Simulation with Densities Known up to a Normalization Constant} \label{sec:MC.simulation} 
Let the uncontrolled process $X$ be the Brownian motion with $b\equiv 0$ and $X_0 = 0$, and set $T = 1$. 
Let  $\phi_\sigma$ denote the density of normal distribution with mean zero and covariance matrix $\sigma^2 I$. 
Given a target distribution with density $f_T$ (which we simply denote by $f$ in this section), the solution to SSB is given by 
\begin{equation}\label{eq:X.A1}
\ud X^*_{t}= u^*_t \ud t + \sigma \ud W_t, 
\end{equation}  
where the  drift $u^*$ is determined by $u^*(x, t) = \nabla\log h(x, t)$ and 
\begin{align*}
    h(x, t) \propto (1-t)^{-d/2} \int \phi_\sigma \left( \frac{z-x}{\sqrt{1 - t}}\right)\cdot\left(\frac{f(z)}{\phi_\sigma(z)}\right)^{\beta/(1 + \beta)}\ud z.  
\end{align*}
Note that to determine $u^*$, we only need to know $f$ up to a normalization constant. 
Since $\nabla \log h = h^{-1} \nabla h $, we can rewrite  $u^*$ as 
\begin{equation}\label{eq:ratio.drift}
  u^*(x,t)  =   \frac{\E_{z \sim \phi_\sigma}[r({x}+\sqrt{1-t} {z}) \nabla \log r({x}+\sqrt{1-t} {z})]}{\E_{{z} \sim \phi_\sigma}[r({x}+\sqrt{1-t} {z})]} 
\end{equation}
where we set  
$$ r(x) = \left(\frac{f(x)}{\phi_\sigma(x)}\right)^{\beta/(1 + \beta)}.$$
We can then approximate the numerator and denominator in~\eqref{eq:ratio.drift} separately using Monte Carlo samples. 

For some target distributions, this approach can be made more efficient by using importance sampling. 
When $f$ has a heavy tail (e.g. Cauchy distribution), $r(x)$ may grow super-exponentially with $x$, and Monte Carlo estimates for the numerator and denominator in~\eqref{eq:ratio.drift} with $z$ drawn from the normal distribution may have large variances. 
Observe that the numerator can be written as 
\begin{align*}
& \E_{z \sim \phi_\sigma}[r({x}+\sqrt{1-t} {z}) \nabla \log r({x}+\sqrt{1-t} {z})]  \\ 
=\;& \int \phi_\sigma(z) r({x}+\sqrt{1-t} {z})^{\beta/(1 + \beta)} \nabla \log r({x}+\sqrt{1-t} {z})  \ud z.  
\end{align*}
The term $ \nabla \log r({x}+\sqrt{1-t} {z})$ is often polynomial in $z$ (that is, it does not grow too fast). 
Hence, intuitively, the integral is likely to be well approximated by a Monte Carlo estimate with $z$ drawn from a density proportional to $\phi_\sigma(z) r({x}+\sqrt{1-t} {z})^{\beta/(1 + \beta)}$. Such a density may not be easily accessible, but if one knows the tail decay rate of $f$, one can try to find a proposal distribution for $z$ with tails not lighter than $\phi_\sigma(z) r({x}+\sqrt{1-t} {z})^{\beta/(1 + \beta)}$. In the experiment given below, we propose $z$ from some distribution with tail decay rate same as $f^{\beta/(1 + \beta)} \phi_\sigma^{1/(1 + \beta)}$.  
Letting $w$ denote our proposal density, we  express the numerator in~\eqref{eq:ratio.drift} as 
\begin{equation}\label{eq:is.num}    
\begin{aligned}
& \E_{z \sim \phi_\sigma}[r({x}+\sqrt{1-t} {z}) \nabla \log r({x}+\sqrt{1-t} {z})] \\ 
=\;& \E_{z \sim w}\left[\frac{\phi_\sigma(z) r({x}+\sqrt{1-t} {z}) }{w(z)} \nabla \log r({x}+\sqrt{1-t} {z})\right], 
\end{aligned}
\end{equation}
and estimate the right-hand side using a Monte Carlo average.  
Similarly, the denominator can be expressed by 
\begin{equation}\label{eq:is.deno}
\E_{{z} \sim \phi_\sigma }[r({x}+\sqrt{1-t} {z})] =  \E_{z \sim w}\left[\frac{\phi_\sigma(z) r({x}+\sqrt{1-t} {z}) }{w(z)}\right]
\end{equation}
and estimated by a Monte Carlo average. 

\subsection{Simulating the Cauchy Distribution} \label{sec:cauchy}
Fix $b \equiv 0$, $\sigma = 1$, $X_0 = 0$ and $T = 1$. 
Let $f$ be the density of the standard Cauchy distribution i.e., $f(x) = \pi^{-1} (1 + x^2)^{-1} $. Theorem~\ref{th:sol.dirac} shows that the solution to SSB is a Schr\"{o}dinger bridge such that the terminal distribution has density  
\begin{equation}\label{eq:f.geom.cauchy}
  f_\beta^* = C_\beta^{-1} \phi(x)^{1/(1 + \beta)} f(x)^{\beta / (1 + \beta)}, 
\end{equation}
where $C_\beta$ is the normalization constant and $\phi$ is the density of the standard normal distribution. 
We plot $f_\beta^*$ and $\log f_\beta^*$ in Figure~\ref{fig:cauchy.geom.densities} for $\beta = \infty, 100, 50, 20$. 
For $x$ close to zero, the density of $f_\beta^*(x)$ remains approximately the same for the four choices of $\beta$. 
When $\beta = \infty$, $f_\beta^*$ is the Cauchy distribution which has a heavy tail. But for any $\beta < \infty$, the tail decay rate of $f_\beta^*$ is dominated by the Gaussian component. 

\begin{figure}[H]
    \centering
    \includegraphics[width=0.38\linewidth]{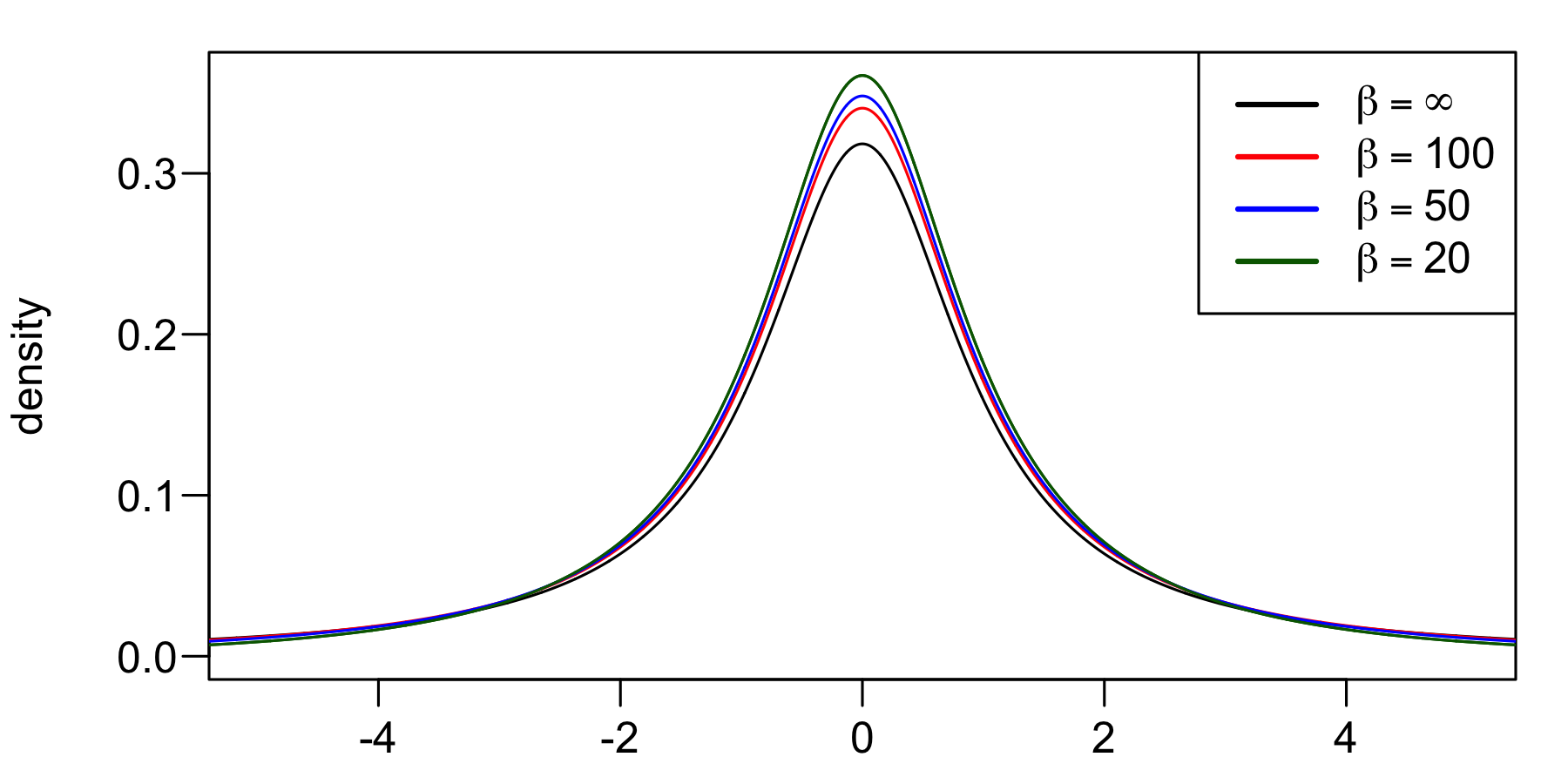}  \hspace{2cm}
    \includegraphics[width=0.38\linewidth]{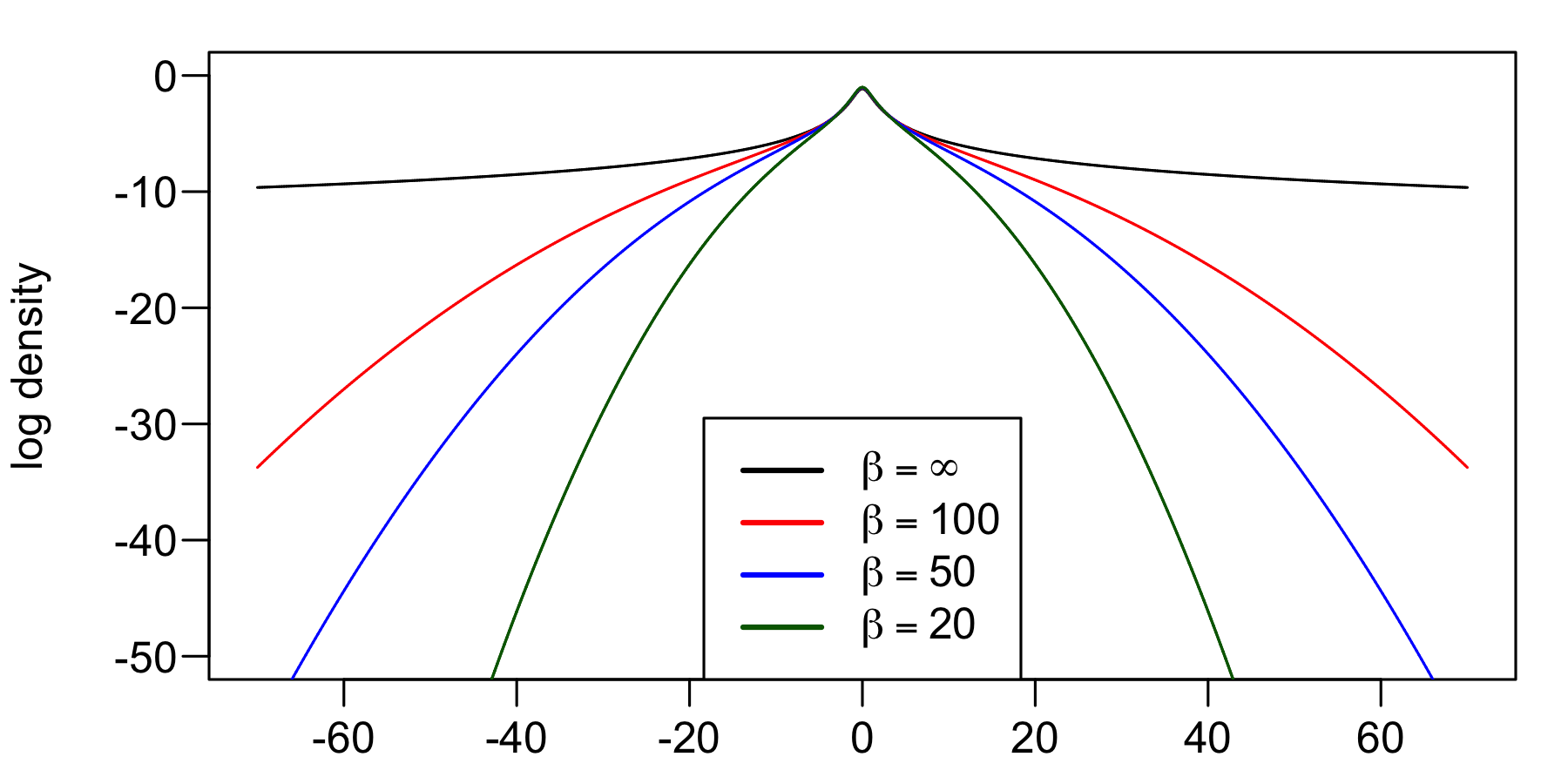} 
    \caption{Densities of  Geometric Mixtures of Normal  Cauchy Distributions.}
    \label{fig:cauchy.geom.densities}
\end{figure}  

We simulate the solution to SSB given in~\eqref{eq:X.A1} over the time interval $[0, 1]$ using the Euler-Maruyama method with $200$ time steps. The drift is approximated by  the importance sampling scheme. 
When $\beta = \infty$ (i.e., the terminal distribution is Cauchy), we let the proposal distribution $w$ in~\eqref{eq:is.num} and~\eqref{eq:is.deno} be the $t$-distribution with $2$ degrees of freedom (we have also tried directly proposing $z$ from the standard Cauchy distribution and obtained very similar results). 
When $\beta < \infty$, we let $w$ be the normal distribution with mean $0$ and variance $1+\beta$. 
We simulate the SSB process 10,000 times, and in Table~\ref{tab:cauchy.fails} we report the number of failed runs; these failures happen because Monte Carlo estimates for the numerator and denominator in~\eqref{eq:ratio.drift} become unstable when $|x|$ is large, resulting in numerical overflow. 
When $\beta = \infty$, we observe that numerical overflow is still common even if we use $1,000$ Monte Carlo samples for each estimate. In contrast, when we use $\beta \leq 100$ and only $200$ Monte Carlo samples, the algorithm becomes very stable. 
In Figure~\ref{fig:cauchy.qqplot}, 
we compare the distribution of generated samples (i.e., the distribution of $X^*_T$ with $T=1$) with their corresponding target distributions. 
The first panel compares the distribution of the samples generated with $\beta = \infty$ (failed runs ignored) with the Cauchy distribution, and the second compares the distribution the samples generated with $\beta = 100$ with the geometric mixture given in~\eqref{eq:f.geom.cauchy}.   
It is clear that simulating the Schr\"{o}dinger bridge process (i.e, using $\beta = \infty$) cannot recover the heavy tails of the Cauchy distribution, but the numerical simulation of SSB with $\beta = 100$ accurately yields samples from the geometric mixture distribution.  
Recall that the KL divergence between standard Cauchy and normal distributions is infinite, which, by Theorem~\ref{th:dp1}, means that  there is no control with finite energy cost that can steer a standard Brownian motion towards the Cauchy distribution at time $T = 1$. 
Our experiment partially illustrates the practical consequences of this fact in numerically simulating the Schr\"{o}dinger bridge, and it also suggests that SSB may be a numerically more robust alternative. 

\begin{table}[H]
    \caption{Number of Failed Attempts in 10,000 Trials. $N_{\rm{mc}}$ denotes the number of Monte Carlo samples used to approximate the expectations defined in~\eqref{eq:is.num} and~\eqref{eq:is.deno}. 
    }
    \label{tab:cauchy.fails}
\begin{center}
    \begin{tabular}{|c|cccccc|}
\hline
$N_{\rm{mc}}$ &   20 & 50 & 100 & 200 & 500 & 1000 \\
\hline
$\beta=\infty$ &   1149 & 933 & 714 & 610 & 444 & 308 \\
$\beta=100$ &   495 & 29 & 6 & 1 & 0 & 0 \\
$\beta=50$ &   124 & 18 & 6 & 0 & 0 & 0 \\
$\beta=20$ &   48 & 4 & 2 & 0 & 1 & 0 \\
\hline
\end{tabular}
\end{center}
\end{table} 
\begin{figure}[H]
    \centering
    \includegraphics[width=0.28\linewidth]{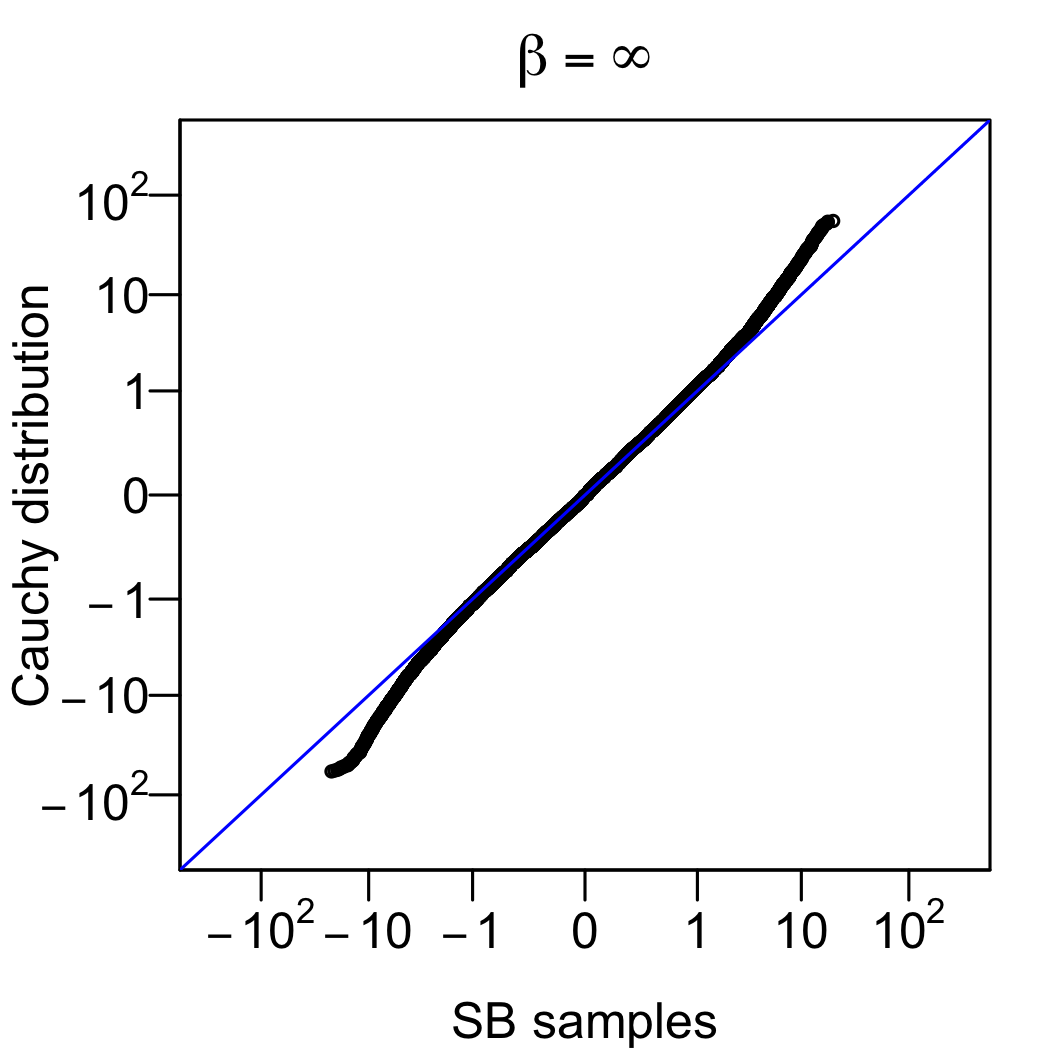}\hspace{2cm} 
    \includegraphics[width=0.28\linewidth]{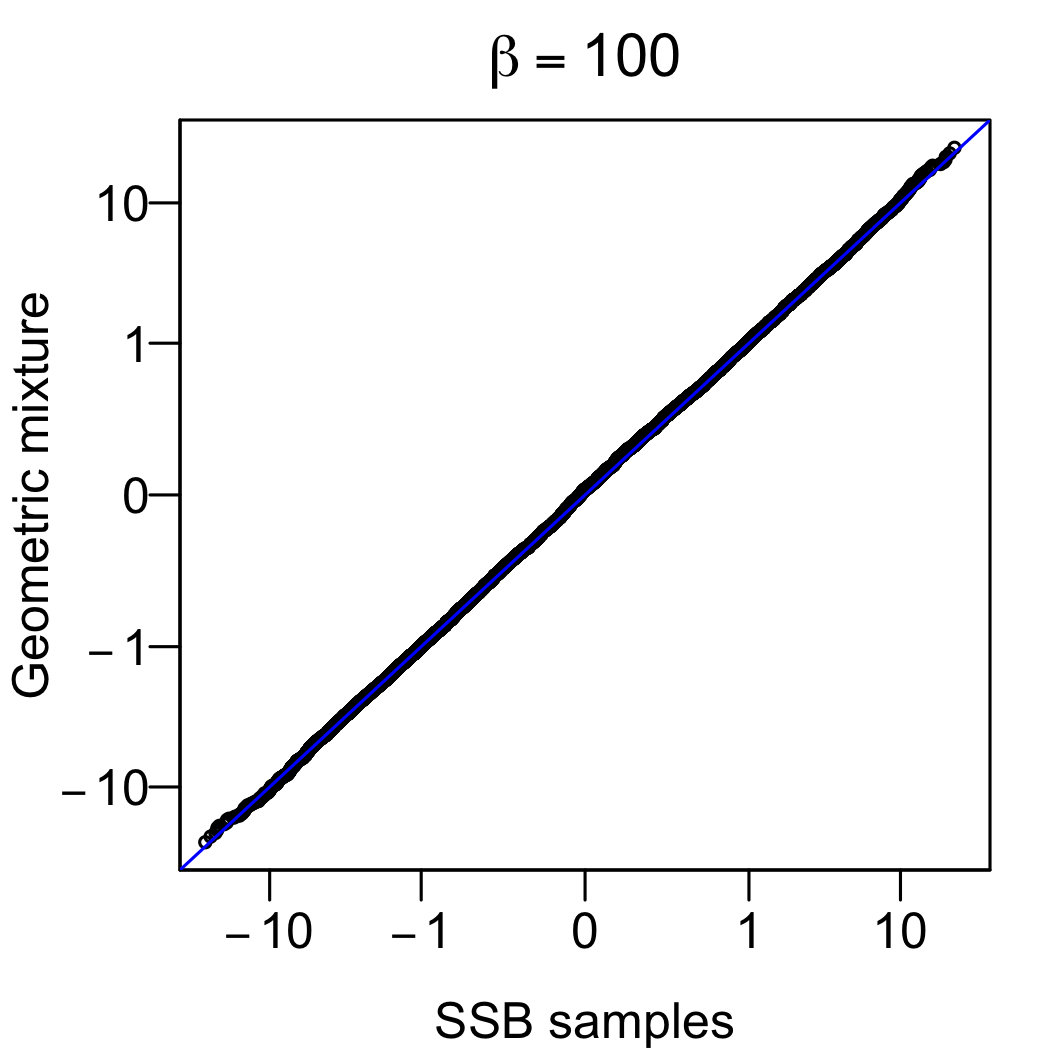}
        \caption{Q-Q plots for the SSB Samples with $N_{\rm{mc}} = 1,000$.}
    \label{fig:cauchy.qqplot}
\end{figure}

\subsection{Simulating Mixtures of Normal Distributions}\label{sec:sim.gaussian.mix}
In Example~\ref{ex:normal.mix}, we let 
$\muR$ be a mixture of four bivariate normal distributions, 
\begin{align*}
    \muR = 0.1 \, \mathcal{N}( (1, 1), \, 0.05^2 I) + 
    0.2 \, \mathcal{N}( (-1, 1), \, 0.05^2 I) + 
    0.3 \, \mathcal{N}( (1, -1), \, 0.05^2 I) +
    0.4 \, \mathcal{N}( (-1, -1), \, 0.05^2 I), 
\end{align*}
where the weights and the mean vector of the four component distributions are different. 
Let $\muT$ be a mixture of two equally weighted bivariate normal distributions
\begin{align*}
    \muT = 0.5 \, \mathcal{N}( (1.2, 0.8), \, 0.5^2 I) + 
    0.5 \, \mathcal{N}( (-1.5, -0.5), \, 0.5^2 I). 
\end{align*}
The first component of $\muT$ has mean close to $(1, 1)$ (which is the mean vector of the first component of $\muR$), and the second component of $\muT$ has mean close to $(-1, -1)$ (which is the mean vector of the last component of $\muR$).  Hence, we can interpret $\muR$  as a distribution of high-quality samples from four different classes, and interpret $\muT$ as a distribution of noisy samples from two of the four classes. 
Let $f_\beta^* = C_\beta^{-1} f_{\REF}(x)^{1/(1 + \beta)} f_{\TGT}(x)^{\beta / (1 + \beta)}$  be the density of our target distribution. 

We generate $1,000$ samples from $f_\beta^*$ by simulating the SDE~\eqref{eq:X.A1} over the time interval $[0, 1]$ with $f = f_\beta^*$ and $\sigma = 1$.  
We use $200$ time steps for discretization and generate $200$ Monte Carlo samples at each step for estimating the drift. 
The trajectories of the simulated processes are shown in Figure~\ref{fig:mix.normal} for $\beta = 0, 2, 10, \infty$, and the samples we generate correspond to $t = 1$. 
It can be seen that when $\beta = 2$ or $10$, the majority of the generated samples form two clusters, one with mean close to $(1, 1)$ and the other with mean close to $(-1, -1)$. Further, the two clusters both exhibit very small within-cluster variation, which indicates that the noise from $\muT$ has been effectively reduced. 

\begin{figure}[H]
\centering 
\includegraphics[width=0.9\linewidth]{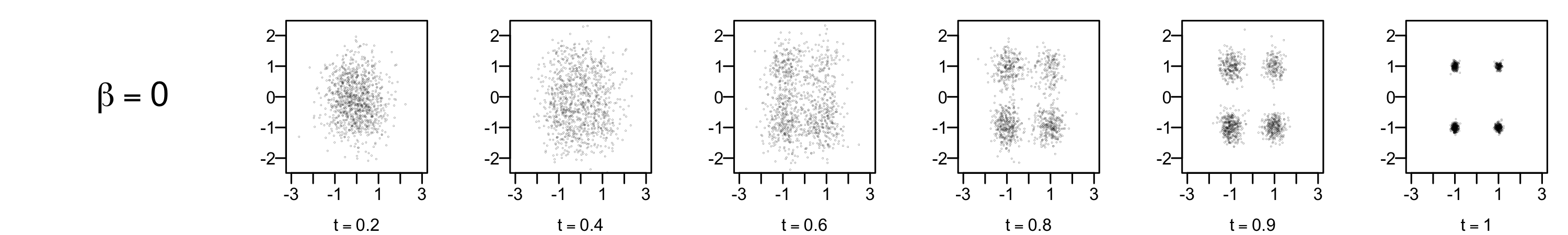} \\
\includegraphics[width=0.9\linewidth]{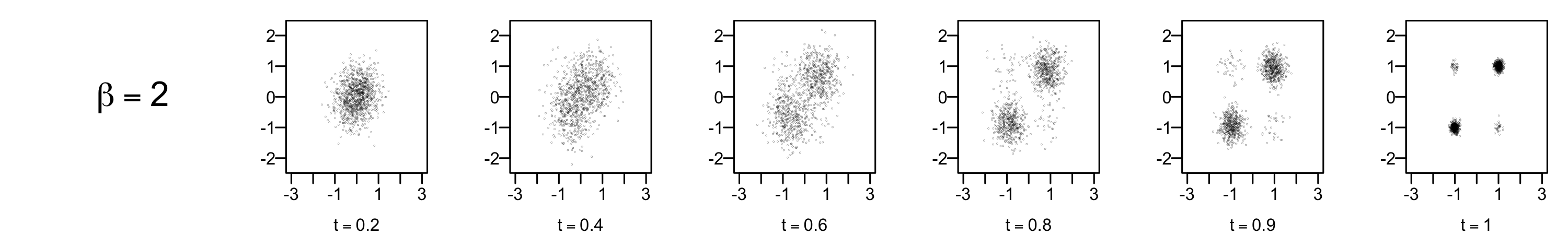} \\
\includegraphics[width=0.9\linewidth]{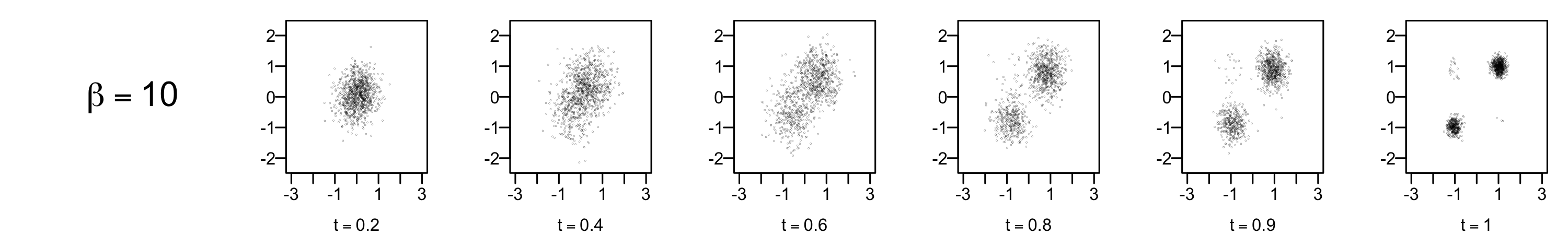} \\ 
\includegraphics[width=0.9\linewidth]{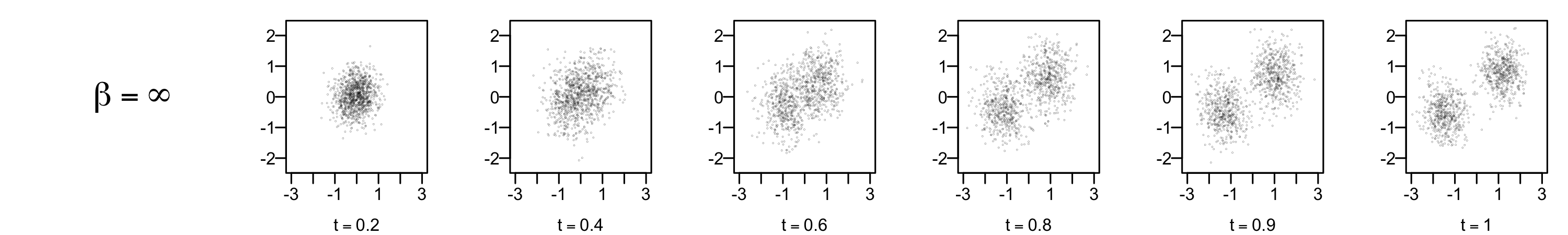}
\caption{SSB Trajectories for Normal Mixture Targets.}
\label{fig:mix.normal}
\end{figure}

\section{Auxiliary Results} \label{sec:aux}
We first present two lemmas about the minimization of Kullback-Leibler divergence. 

\begin{lemma}\label{lm:min.KL1} 
Let $(\Omega, \cF, \pi)$ be a $\sigma$-finite measure space, and let $\nu \ll \pi$ be a  measure with density $f = \ud \nu / \ud \pi$ such that $ C = \nu(\Omega) \in (0, \infty)$. 
Let $\cP_\pi$ denote the set of all probability measures absolutely continuous with respect to $\pi$. 
Then, 
\begin{align*}
   \inf_{\mu \in \cP_\pi} \KL( \mu, \nu) = - \log C. 
\end{align*}
The infimum is attained by the probability measure $\mu^*$ such that $\ud \mu^* /  \ud \pi = C^{-1} f(x)$.  
\end{lemma} 
\begin{proof}[Proof of Lemma~\ref{lm:min.KL1}]
Since $\KL(\mu, \nu) = \infty$ if $\mu \not\ll \nu$, it suffices to consider $\mu \in \cP_\pi$ such that $\mu \ll \nu$. 
It is straightforward to show that $\KL( \mu, \nu) = \KL( \mu, \mu^*) -  \log C$. 
Since $\int \ud \mu^* = C^{-1} \int f \ud \pi = C^{-1} \int \ud \nu = 1$, $\mu^*$ is a probability measure.   
The claim then follows from the fact that the KL divergence between any two probability measures is non-negative. 
\end{proof}

\begin{lemma}\label{lm:min.KL} 
Let $(\Omega, \cF, \pi)$ and $\cP_\pi$ be as given in Lemma~\ref{lm:min.KL1}. 
For $i = 0, 1$, let $\nu_i$ be a finite measure (i.e., $\nu_i(\Omega) < \infty$) with density $f_i = \ud \nu_i / \ud \pi$. Assume that $\nu_1 \ll \nu_0$. 
For $\beta \geq 0$, 
\begin{align*}
   \inf_{\mu \in \cP_\pi} \KL( \mu, \nu_0) + \beta\KL(\mu, \nu_1) = - (1 + \beta) \log C_\beta,  
\end{align*}
where $C_\beta = \int f_0(x)^{1/(1 + \beta)} f_1(x)^{\beta/(1 + \beta)} \pi (\ud x) \in (0, \infty)$.  
The infimum is attained by the probability measure $\mu^*_\beta$ such that 
$$f^*_\beta = \frac{ \ud \mu^*_\beta }{\ud \pi }=  C_\beta^{-1} f_0^{1/(1 + \beta)} f_1^{\beta/(1 + \beta)}.$$
When $\nu_0, \nu_1 \in \cP_\pi$, we have $C_\beta \in (0, 1]$. Further, 
\begin{align*}
    \lim_{\beta \rightarrow \infty} C_\beta  = \nu_1(\Omega), \quad \lim_{\beta \rightarrow \infty} \KL( \mu^*_\beta, \nu_0) = \KL(\nu_1, \nu_0). 
\end{align*}
\end{lemma}

\begin{proof}[Proof of Lemma~\ref{lm:min.KL}]
Since $\nu_1 \ll \nu_0$, we have $C > 0$. 
By H\"{o}lder's inequality, 
$$C_\beta \leq  \left( \int f_0 \ud \pi \right)^{1 / (1 + \beta)} \left( \int f_1 \ud \pi \right)^{\beta / (1 + \beta)} = 
\nu_0(\Omega)^{1 / (1 + \beta)}\nu_1(\Omega)^{\beta / (1 + \beta)} < \infty.$$ 
Clearly, if $\nu_0, \nu_1 \in \cP_\pi$, the above inequality implies $C \leq 1$. 
Observe that 
\begin{align*}
     \KL( \mu, \nu_0) + \beta\KL(\mu, \nu_1) = (1 + \beta) \KL(\mu, \nu)
\end{align*}
where $\nu$ is the measure with density $\ud \nu / \ud \pi = f_0^{1/(1 + \beta)} f_1^{\beta/(1 + \beta)} = C_\beta f^*_\beta$.  Hence, we can apply Lemma~\ref{lm:min.KL1}  to prove that $\mu^*_\beta$ is the minimizer. 

To prove the convergence as $\beta \rightarrow \infty$, let $A = \{x \colon f_0(x) \geq f_1(x) \}$ and write $C_\beta = C_{\beta,0} + C_{\beta,1}$, where 
\begin{align*}
    C_{\beta,0} = \int_A f_0(x)^{1/(1 + \beta)} f_1(x)^{\beta/(1 + \beta)} \pi (\ud x), \quad \quad 
    C_{\beta,1} = \int_{\Omega \setminus A} f_0(x)^{1/(1 + \beta)} f_1(x)^{\beta/(1 + \beta)} \pi (\ud x).
\end{align*}
The integrands of both $C_{\beta,0}$ and $C_{\beta,0}$ are monotone in $\beta$. Hence, using $C_0 < \infty$ and monotone convergence theorem, we find that  
$$\lim_{\beta \rightarrow \infty} C_\beta = \int \lim_{\beta \rightarrow \infty}f_0(x)^{1/(1 + \beta)} f_1(x)^{\beta/(1 + \beta)} \pi (\ud x) = \int f_1(x) \pi (\ud x) = \nu_1(\Omega), $$
which also implies $f^*_\beta \rightarrow f_1 / \nu_1(\Omega)$ pointwise. An analogous argument using monotone convergence theorem proves that $\lim_{\beta \rightarrow \infty} \KL( \mu^*_\beta, \nu_0) = \KL(\nu_1, \nu_0)$. 
\end{proof}

The next result is about the controlled SDE~\eqref{eq:CSDE} with $u = \sigma^2 \nabla \log h$. It is adapted from  Theorem 2.1 of~\citet{dai1991stochastic}. Our proof is provided for completeness. 

\begin{theorem}\label{th:h}
Suppose Assumption~\ref{ass.b} holds. Let $X = (X_t)_{\subT}$ be a weak solution to~\eqref{eq:udp} with initial distribution $\mu_0$ and transition density $p(x, t \mid y, s)$. 
Define 
\begin{equation*}
    h(x, t) = \int  g(z) p(z, T \mid x, t) \ud z, \quad \text{for } (x, t) \in \R^d \times [0, T), 
\end{equation*}   
for some measurable $g \geq 0$   such that $\E g(X_T) < \infty$. 
Assume  $h > 0$ on $\R^d \times [0, T)$.   
Let $X^h = (X^h_t)_{\subT}$  be a weak solution with $X^h_0 = X_0$ to the SDE 
\begin{equation}\label{eq:rt.SDE} 
\ud X_t^h = \left[ b(X_t^h, t) + \sigma^2 \nabla \log h (X_t^h, t) \right] \ud t + \sigma \ud W_t, \text{ for } t \in [0, T]. 
\end{equation} 
Then, we have the following results.
\begin{enumerate}[(i)]
     \item $h \in \C^{2, 1}(\R^d \times [0, T) )$  and 
\begin{equation}\label{eq:h.pde}
   \frac{\partial h}{\partial t} + \sum_{i = 1}^d b_i \frac{\partial h}{\partial x_i} + \frac{\sigma^2}{2} \sum_{i=1}^d \frac{\partial^2 h}{\partial x_i^2} = 0, \text{ for } (x, t) \in \R^d \times [0, T).  
\end{equation} 
    \item The weak solution $X^h$ to the SDE~\eqref{eq:rt.SDE} exists. Indeed, we can define a probability measure $\Q$ by $\ud \Q / \ud \P = g(X_T)/h(X_0, 0)$ such that he law of $X$ under $\Q$ is the same as the law of $X^h$ under $\P$. 
    \item The process $X^h$ satisfies 
    \begin{equation}
        \E \left[  \log \frac{ g(X^h_T) }{h(X_0^h, 0)} \right] = \E\left[   \int_0^{T}  \frac{\sigma^2}{2} \| \nabla \log h (X_s^h, s)\|^2 \ud s \right]. 
    \end{equation}
    \item  The transition density of $X^h$ is given by 
    \begin{equation}\label{eq:def.ph}
        p_h(x, t \mid y, s) = \frac{p(x, t \mid y, s) h (x, t)}{h(y, s)}, \quad \text{ for } 0 \leq s < t \leq T.
    \end{equation}
    Hence, $X^h$ is a Doob's $h$-path process, and the density of the distribution of $X^h_T$ is  
    \begin{equation}\label{eq:distr.XT} 
   \frac{ \ud  \Law(X^h_T)}{\ud \lambda}(x) = g(x) \int  \frac{p(x, T \mid y, 0)}{h(y, 0)}   \mu_0(\ud y). 
\end{equation}
\end{enumerate}
\end{theorem}

\begin{proof}[Proof of Theorem~\ref{th:h}]
We follow the arguments of~\citet{jamison1975markov, dai1991stochastic}. 
Under Assumption~\ref{ass.b}, Proposition 2.1 of~\citet{dai1991stochastic} (which is adapted from the result of~\cite{jamison1975markov}) implies that $h \in \C^{2, 1}(\R^d \times [0, T) )$ and $\cL h + \frac{\partial h}{\partial t} = 0$ on $\R^d \times [0, T)$, where 
\begin{equation} \label{eq:generator}
\cL  = \sum_{i = 1}^d b_i \frac{\partial }{\partial x_i} + \frac{\sigma^2}{2} \sum_{i=1}^d \frac{\partial^2}{\partial x_i^2} 
\end{equation} 
denotes the generator of $X$. This proves part (i).   

To prove part (ii),  we first apply It\^{o}'s lemma to get 
\begin{align*}
    \log h(X_t, t) =\;& \log h(X_0, 0) + \int_0^t \tilde{b}(X_s, s) \ud s  
    + \int_0^t  \sigma \nabla \log h(X_s, s) \ud W_s,  
\end{align*}
for any $t \in [0, T)$, where 
\begin{align*}
    \tilde{b}(x, t) = \frac{1}{h(x, t)} \left((\cL h )(x, t)  + \frac{\partial h}{\partial t} (x, t) \right) - \frac{\sigma^2}{2} \| \nabla \log h (x, t)\|^2 = - \frac{\sigma^2}{2} \| \nabla \log h (x, t)\|^2.  
\end{align*}
Since $g(X_T)$ is integrable,  $h(X_t, t)$ is a uniformly integrable martingale on $[0, T)$ and converges to $g(X_T)$ both a.s. and in $L^1$. Letting $t \uparrow T$, we obtain that 
\begin{equation}\label{eq:ito.g}
    \log g(X_T) =  \log h(X_0, 0) - \int_0^T  \frac{\sigma^2}{2} \| \nabla \log h (X_s, s)\|^2 \ud s  
    + \int_0^T  \sigma   \nabla \log h(X_s, s) \ud W_s.   
\end{equation}
Write $h(x, T) = g(x)$, $Y_t = h(X_t, t)$ and $Z_t = Y_t / Y_0$. We have shown that $Y_t$ and $Z_t$ are  martingales on $[0, T]$. Since $\E[Z_T] = 1$, we can define a probability measure $\Q$ by $\ud \Q / \ud \P = Z_T$. 
By Girsanov theorem and the expression for $Z_T$ given in~\eqref{eq:ito.g},  the law of $X$ under $\Q$ is the same as the law of $X^h$ under $\P$; in other words, $\ud \P^h_X  = (g(x_T) / h(x_0, 0)) \ud \P_X$, where $\P_X$ (resp. $\P_X^h$) is the probability measure induced by $X$ (resp. $X^h$) on the space of continuous functions.

For part (iii), choose $t_n = T \wedge \tau_n$, where $\tau_n =\inf \{ t \colon |X^h_t| \geq n\}$.  
Analogously to~\eqref{eq:ito.g}, we can apply It\^{o}'s lemma to get 
\begin{align*} 
        \log h(X^h_{t_n}, {t_n}) =  \log h(X^h_0, 0)  + \int_0^{t_n} \frac{\sigma^2}{2} \| \nabla \log  h(X^h_s, s) \|^2 \ud s  + \int_0^{t_n}  \sigma   \nabla \log h(X^h_s, s) \ud W_s. 
\end{align*} 
Since $h$ is smooth, $| \nabla \log h(X^h_s, s) |$ is bounded on $[0, t_n]$. Taking expectations on both sides,  we find that 
\begin{equation} 
\E \left[  \log h(X^h_{t_n}, {t_n})  \right] = \E\left[  \log h(X^h_0, 0) +  \int_0^{t_n}  \frac{\sigma^2}{2} \| \nabla \log h(X^h_s, s)  \|^2 \ud s \right]. 
\end{equation}
Letting $n \rightarrow \infty$ and applying monotone convergence theorem, we get 
\begin{equation}\label{eq:cost.identity}
\liminf_{n \rightarrow \infty} \E \left[  \log h(X^h_{t_n}, {t_n})  \right] 
= \E\left[  \log h(X^h_0, 0) +  \int_0^{T}  \frac{\sigma^2}{2} \| \nabla \log h(X^h_s, s)  \|^2 \ud s \right]. 
\end{equation}
It remains to argue that the left-hand side converges to $ \E \left[  \log h(X^h_{T}, T)  \right]$.  
Write $Y^h_t = h(X^h_t, t)$. Using the change of measure and $h(x, T) = g(x)$, we get 
\begin{equation*} 
    \E\left[  \log Y^h_t \right] = \E \left[ \frac{ Y_t }{Y_0} \log Y_t \right].   
\end{equation*}
The function $f(y) = y \log y$ is bounded below and convex. 
If $(t_n)$ is chosen such that $t_n \uparrow T$,  we have 
\begin{align*}
     \E \left[ \lim_{n \rightarrow \infty}  Y_{t_n} \log Y_{t_n} \mid \cF_0 \right]  \leq \liminf_{n \rightarrow \infty}  \E \left[ Y_{t_n} \log Y_{t_n} \mid \cF_0  \right]   \leq \E \left[  Y_T \log  Y_T \mid \cF_0 \right], 
\end{align*}
where Fatou's lemma is applied to obtain the first inequality, and the second inequality follows from the fact that $Y_t \log Y_t$ is a submartingale. Since $Y_t$ converges a.s. to $Y_T = g(X_T)$, we get 
\begin{equation}
    \liminf_{n \rightarrow \infty}  \E \left[ Y_{t_n} \log Y_{t_n} \mid \cF_0  \right]   = \E \left[  Y_T \log  Y_T \mid \cF_0 \right]. 
\end{equation}
Taking expectations on both sides we get 
\begin{equation}\label{eq:xlogx.converge}
        \liminf_{n \rightarrow \infty}  \E \left[ \log Y^h_{t_n} \right]   = \E \left[ \log Y^h_{T} \right]. 
\end{equation}
Combining it with~\eqref{eq:cost.identity} proves part (iii).

Consider part (iv). For any bounded and measurable function $f$, we apply the change of measure to the conditional expectation to get 
\begin{align*}
 \E_{\Q}[ f(X_t) \mid \cF_s ] = \frac{ \E[ f(X_t) h(X_t, t) \mid \cF_s] }{ h(X_s, s) }.  
\end{align*}
The claim then follows from Fubini's theorem. 
\end{proof}

\section{Proofs for the Main Results} \label{sec:proofs}
Recall that $X$ denotes the solution to the SDE 
\begin{equation}  
\ud X_t= b(X_t, t) \ud t +  \sigma \ud W_t
\end{equation} 
over the time interval $[0, T]$ with initial distribution $\Law(X_0) = \mu_0$, and $X^u$ denotes the solution to the controlled SDE
\begin{equation} 
\ud X^u_t=\left[ b (X^u_t, t )+ u_t \right] \ud t + \sigma \ud W_t, 
\end{equation} 
with $X_0^u = X_0$. The transition density of the uncontrolled process $X$ is denoted by $p(x, t \mid y, s)$.

\begin{proof}[Proof of Lemma~\ref{lm:girsanov}]
Since $u$ is admissible, it must satisfy $\E \int_0^T \|u_t\|^2 \ud t < \infty$. 
Therefore, Novikov's condition is satisfied, and we can apply Girsanov theorem  to get
\begin{align*}
    \E \left[ \frac{g(X_T)}{h(X_0, 0)} \right] =\;& \E\biggr[  \frac{g(X_T^u)}{h(X_0^u, 0)} \exp\biggr\{ 
    -\int_0^T \frac{u_t}{\sigma} \ud W_t - \int_0^T \frac{ \|u_t\|^2}{2 \sigma^2} \ud t \biggr\} \biggr] \\
    \geq\;& \exp\biggr\{ \E\biggr[  \log \frac{g(X_T^u)}{h(X_0^u, 0)} 
    -\int_0^T \frac{u_t}{\sigma} \ud W_t-   \int_0^T \frac{ \|u_t\|^2}{2 \sigma^2} \ud t   \biggr] \biggr\}.
\end{align*} 

By part (ii) of Theorem~\ref{th:h}, the left-hand side of the above inequality is equal to $1$. Hence, 
\begin{align*}
   & 0 \geq \E\left[  \log \frac{g(X_T^u)}{h(X_0^u, 0)} 
    -\int_0^T \frac{u_t}{\sigma} \ud W_t - \int_0^T \frac{ \|u_t\|^2}{2 \sigma^2} \ud t   \right]  = \E\left[  \log \frac{g(X_T^u)}{h(X_0^u, 0)} 
  - \int_0^T \frac{ \|u_t\|^2}{2 \sigma^2} \ud t   \right], 
\end{align*} 

where we have used $\E \int_0^T \| u_t\|^2 \ud t < \infty$ again to obtain $\E[  \int_0^T u_t \ud W_t   ] = 0$. 
By the definition of $J_\beta$, we have 
\begin{align*}
    J_\beta(u) = \;& \beta \, \KL (\Law(X^u_T), \mu_T) + \E\left[  \int_0^T \frac{ \|u_t\|^2}{2 \sigma^2} \ud t    \right]    \\
    \geq  \;& \beta \, \KL (\Law(X^u_T), \mu_T) + \E\left[  \log \frac{g(X_T^u)}{h(X_0^u, 0)}   \right]. 
\end{align*} 
Since $\Law(X_0^u) = \mu_0$, we have $\E[ \log h(X_0^u, 0) ] = \int \log h(x, 0) \mu_0(\ud x)$. 
By part (iii) of Theorem~\ref{th:h}, the equality is attained when $u_t = (\cT g)(X^u_t, t)$.   
\end{proof}

\begin{proof}[Proof of Theorem~\ref{th:sol.RSB}]
First, we find using~\eqref{eq:rsb2}  that 
$$\E \rho_T(X_T) =  \int   \frac{ \ud \mu_0 }{ \ud \nu_0 } \ud \mu_0,$$
which is finite by assumption. Since $p(z, T \mid x, t) > 0$ whenever $t < T$, we have $h > 0$ on $\R^d \times [0, T)$. 
Hence, Theorem~\ref{th:h} and Lemma~\ref{lm:girsanov} can be applied with $g = \rho_T$.  
In addition to setting $\rho_T = \ud \nu_T / \ud \lambda$ and $f_T = \ud \mu_T / \ud \lambda$, we define  
\begin{equation}\label{eq:densities}
\begin{aligned}
   k_T(x) =  \int p(x, T \mid y, 0) \nu_0(\ud y), \quad 
   q_T^u(x) =  \frac{\ud \Law(X^u_T)}{\ud \lambda}(x). 
\end{aligned}
\end{equation}
There is no loss of generality in assuming $q_T^u$ exists, since $\mu_T \ll \lambda$ and  $\KL (\Law(X^u_T), \mu_T) = \infty$ if $\Law(X_T^u) \not\ll \mu_T$.  
For later use, we note that by~\eqref{eq:rsb2} and~\eqref{eq:rsb1}, 
\begin{align}
    &h(y, 0) =  \frac{\ud \mu_0}{\ud \nu_0}(y), \label{eq:rsb3} \\
    &\rho_T(x) =  \left( \frac{ f_T(x) }{ k_T(x) }  \right)^{\beta/(1 + \beta)}.  \label{eq:rsb4}
\end{align} 

Let $\eta$ be the $\sigma$-finite measure on $\R^d$ with density $f_T^{\beta/(1+\beta)} k_T^{1/(1+\beta)}$. 
By Lemma~\ref{lm:girsanov}, for any admissible control $u$, 
\begin{align*}
    J_\beta(u) \geq \;&   \beta \, \KL (\Law(X^u_T), \mu_T)   + \E\left[  \log   \rho_T(X_T^u) \right] - \int \log h(x, 0) \mu_0(\ud x) \\ 
     = \;&  \beta \, \KL (\Law(X^u_T), \mu_T)  +\E\left[  \log   \rho_T(X_T^u) \right] - \KL(\mu_0, \nu_0) \\
     = \;&   \E\left[ \beta \log \frac{ q_T^u(X^u_T)}{ f_T(X^u_T) }   + \frac{\beta}{1 + \beta} \log \frac{ f_T(X^u_T) }{k_T(X^u_T) }  \right]   -   \KL(\mu_0, \nu_0) \\
     = \;& \beta \KL( \Law(X^u_T),  \eta ) - \KL(\mu_0, \nu_0), 
\end{align*}
where the second line follows from~\eqref{eq:rsb3} and the third from~\eqref{eq:rsb4}. 
The measures $\mu_0, \nu_0, \eta$ do not depend on $u$. By Lemma~\ref{lm:min.KL1},  
\begin{equation}
\begin{aligned}
    \KL( \Law(X^u_T),  \eta ) \geq - \log C, \quad \text{where }  C = \int  f_T(x)^{\beta/(1+\beta)} k_T(x)^{1/(1+\beta)} \ud x. 
\end{aligned}
\end{equation}
We will later prove that $C = 1$. 
Combining the above two displayed inequalities, we get 
\begin{align}
     J_\beta(u) \geq\;&  \beta \KL( \Law(X^u_T),  \eta ) -   \KL(\mu_0, \nu_0) \label{eq:Jbeta2} \\
     \geq \;&  - \log C - \KL(\mu_0, \nu_0). \label{eq:Jbeta3}
\end{align}
For $u^*_t = \sigma^2 \nabla \log h(X_t^{u^*}, t)$, we know by Lemma~\ref{lm:girsanov} that the equality in~\eqref{eq:Jbeta2} is attained. Hence, it is optimal if we can show that the equality in~\eqref{eq:Jbeta3} is also attained. By Lemma~\ref{lm:min.KL1}, this is equivalent to showing that 
\begin{equation}\label{eq:opt.qT}
    q_T^*(x) = C^{-1} f_T(x)^{\beta/(1+\beta)} k_T(x)^{1/(1+\beta)}, 
\end{equation} 
where we write $q_T^* = q_T^{u^*}$. 
By part (iv) of Theorem~\ref{th:h}, we have 
\begin{align*}
    q_T^*(x) =\;&  \rho_T(x) \int  \frac{p(x, T \mid y, 0)}{h(y, 0)}   \mu_0(\ud y) \\
    \overset{(i)}{=}\;&  \rho_T(x)  \int p(x, T \mid y, 0)   \nu_0(\ud y)  \\
    =\;& \rho_T(x)  k_T(x)  \\
    \overset{(ii)}{=}\;&   f_T(x)^{\beta/(1 + \beta)} 
    k_T(x)^{1/(1 + \beta)}, 
\end{align*}
where step $(i)$ follows from~\eqref{eq:rsb3} and step (ii) follows from~\eqref{eq:rsb4}.  
So $u^*$ is optimal, and the normalizing constant $C$ in~\eqref{eq:opt.qT} equals $1$,  from which it follows that $J_\beta(u^*) = -\KL(\mu_0, \nu_0)$. 
Finally, one can apply Jensen's inequality and the assumption $\int ( \ud \mu_0 / \ud \nu_0 ) \ud \mu_0 < \infty$ to show that $|\KL(\mu_0, \nu_0)| < \infty$, which concludes the proof. 
\end{proof}

\begin{proof}[Proof of Theorem~\ref{th:sol.ts}]
Recall that $\nu_0, \nu_N$ are defined by 
\begin{align}
   \frac{\ud \mu_0}{\ud \nu_0}(y) =\;& \int p_N(\bx_N \mid y) \nu_N(\ud \bx_N), \label{eq:tsb2} \\ 
   \frac{\ud \mu_N}{\ud \lambda}(\bx_N) =\;& \rho_N(\bx_N)^{\frac{ 1+\beta }{\beta}} \int p_N(\bx_N \mid y) \nu_0(\ud y). \label{eq:tsb1} 
\end{align}
The proof is essentially the same as that of Theorem~\ref{th:sol.ts}. 
First, we need to derive a result analogous to Lemma~\ref{lm:girsanov}, which we give in Lemma~\ref{lm:girsanov.ts} below. 
We will apply Lemma~\ref{lm:girsanov.ts} with $g = \rho_N$.  Recall that $h_j$ is defined by 
\begin{equation}\label{eq:def.hj2}
\begin{aligned}
    h_j(x, t; \bx_j) = \E\left[\rho_N(\bx_{j}, X_{t_{j+1}}, \dots, X_{t_N} ) \,\Big|\, X_t = x \right], \text{ for } (x, t) \in \R^d \times [t_j, t_{j + 1}). 
\end{aligned}
\end{equation} 
Note that the conditions $\E \rho_N(\bX_N) < \infty$ and $h_j > 0$ can be verified by the same argument as that used in the proof of Theorem~\ref{th:sol.RSB}. 
By~\eqref{eq:tsb2}, we have 
\begin{equation}\label{eq:tsb3}
 h_0(y, 0) = \E[ \rho_N(\bX_N) \mid X_0 = y] = \frac{\ud \mu_0}{\ud \nu_0}(y). 
\end{equation}
Define 
\begin{equation}\label{eq:densities2}
\begin{aligned}
   f_N(\bx_N) = \frac{\ud \mu_N }{\ud \lambda}(\bx_N), \quad k_N(\bx_N) = \int p_N(\bx_N \mid y) \nu_0(\ud y), \quad q_N^u(\bx_N) = \frac{\ud \Law(\bX^u_N)}{\ud \lambda}(\bx_N).
\end{aligned}
\end{equation} 
We can rewrite~\eqref{eq:tsb1} as $\rho_N = ( f_N / k_N )^{\beta/(1 + \beta)}$. 

By Lemma~\ref{lm:girsanov.ts} and Lemma~\ref{lm:min.KL1}, for any admissible control $u$,
\begin{align*} 
    J_\beta^N(u) \geq  - \log C  -   \KL(\mu_0, \nu_0), \quad \text{where } C = \int  f_N(\bx_N)^{\beta/(1+\beta)} k_N(\bx_N)^{1/(1+\beta)} \ud \bx_N.  
\end{align*} 
To prove that the equality is attained by the control $u_t^* = (\cT_N \rho_N)(X^{u^*}_t, t, \bX^{u^*}_N)$, it remains to show that 
\begin{equation} 
    q_N^*(\bx_N) = C^{-1} f_N(\bx_N)^{\beta/(1+\beta)} k_N(\bx_N)^{1/(1+\beta)}. 
\end{equation}
where $q_N^* = q_N^{u^*}$. To find $q_N^*$, we can mimic the proof of Theorem~\ref{th:h}. 
It is not difficult to verify that the law of $X^*$ under $\P$ is the same as the law of $X$ under $\Q$, where $\Q$ is defined by 
\begin{align}
     \frac{\ud \Q}{\ud \P} =  \frac{ \E[ \rho_N( \bX_N ) \mid \cF_T ] }{ \E[ \rho_N( \bX_N )  \mid \cF_0] }  
    = \frac{  \rho_N(\bX_N) }{ h_0(X_0, 0) }.
\end{align}
So for any bounded and measurable function $\ell$, we have
\begin{align*}
    &\E_{\Q} [ \ell( \bX_N) ] = \E \left[  \ell( \bX_N) \frac{\rho_N(\bX_N)  }{ h_0(X_0, 0)} \right] 
    = \int  \ell (\bx_N)  \rho_N(\bx_N) \left\{ \int \frac{ p_N(\bx_N \mid y) }{h_0(y, 0)} \mu_0(\ud y) \right\} \ud \bx_N.
\end{align*}
It thus follows from~\eqref{eq:tsb3} that
\begin{equation}
    q_N^*(x) = \rho_N(\bx_N) \int  \frac{p_N(\bx_N \mid y)}{h_0(y, 0)}   \mu_0(\ud y)  
    =   f_N(\bx_N)^{\beta/(1 + \beta)} 
    k_N(\bx_N)^{1/(1 + \beta)}. 
\end{equation} 
The rest of the proof is identical to that of Theorem~\ref{th:sol.RSB}. 
\end{proof}

\begin{lemma}\label{lm:girsanov.ts}
Let $u$ be an admissible control and $g \colon \R^{d \times N} \rightarrow [0, \infty)$ be a measurable function such that $\E[ g(\bX_N) ] < \infty$.  
Let $h_j$ be as given in~\eqref{eq:def.hj}, and assume $h_j > 0$ for each $j$. 
For the cost $J_\beta^N$ defined in~\eqref{eq:value.SB2}, we have  
\begin{align*}
    J_\beta^N(u) \geq \beta \, \KL (\Law(\bX^u_N), \mu_N) +  
 \E[\log g( \bX^u_N)] - \int \log h_0(x, 0) \mu_0(\ud x). 
\end{align*} 
The equality is attained when $u_t = (\cT_N g)(X^u_t, t, \bX^u_N)$. 
\end{lemma}

\begin{proof}[Proof of Lemma~\ref{lm:girsanov.ts}]
The SDE~\eqref{eq:CSDE} with control $u_t^* = (\cT_N g)(X^{u^*}_t, t, \bX^{u^*}_N)$ can be expressed as  
\begin{equation}\label{eq:h.sde.ts}
\begin{aligned}
    &X_0^* = X_0, 
    \\& \ud X_t^* = \left[ b(X_t^*, t) + \sigma^2 \nabla \log h_j (X_t^*, t; \bX_j^*) \right] \ud t +\sigma \ud W_t, \text{ for } t \in [t_j, t_{j+1}), 
\end{aligned}
\end{equation} 
where we write $X^* = X^{u^*}$ and $h_j$ is as given in~\eqref{eq:def.hj}.  
By the tower property,  we have 
\begin{equation}\label{eq:tower.hj}
    h_j(x, t; \bx_j) = \E\left[  h_{j+1}(X_{t_{j+1}}, t_{j+1}; \bx_j, X_{t_{j+1}}) \mid X_t = x \right], 
\end{equation} 
where $h_N$ is defined by $h_N(x, t_N; \bx_{N-1}, x) = g(\bx_{N-1}, x)$, and $X$ is the solution to the uncontrolled process~\eqref{eq:udp}.   
The assumption $\E[ g( \bX_N) ] < \infty$ implies that $\E[h_j( X_{t_j}, t_j; \bx_{j-1}, X_{t_j})] < \infty$ for almost every $\bx_{j-1}$. 
Hence, for each $j$ and $x_j \in \R^d$,  Theorem~\ref{th:h} implies that there exists a weak solution to the SDE~\eqref{eq:h.sde.ts} on $[t_j, t_{j + 1}]$ with $X_{t_j}^* = x_j$. 
Moreover, 
\begin{equation} 
\begin{aligned}
\E \left[  \log \frac{ h_{j+1}( X_{t_{j+1}}^*, t_{j+1}; \bX_{j+1}^*) }{h_j(X_{t_j}^*, t_j; \bX_j^*)} \right] = \E\left[    \int_{t_j}^{t_{j+1}}  \frac{\sigma^2}{2} \| \nabla \log h_j (X_s^*, t; \bX_j^*)  \|^2 \ud s \right].  
\end{aligned}
\end{equation} 
Summing over all the $N$ time intervals, we get
\begin{align*} 
 \E \left[  \log \frac{ g(\bX^*_N ) }{h_0(X_0^*, 0)}  \right] 
= \E\left[    \sum_{j = 0}^{N-1}  \log \frac{ h_{j+1}( X_{t_{j+1}}^*, t_{j+1}; \bX_{j+1}^*) }{h_j(X_{t_j}^*, t_j; \bX_j^*)}   \right]   
= \E\left[     \int_0^T  \frac{ \|  u_t^*  \|^2 }{2 \sigma^2 }  \ud s \right].  
\end{align*}

Now consider an arbitrary admissible control $u$. As in the proof of Lemma~\ref{lm:girsanov}, by Girsanov theorem, we have 
\begin{align*}
    1 =\;& \E\left[ \frac{g(\bX_{N})}{h_0(X_0, 0)} \right] \\
    =\; &  \E\biggr[  \frac{g(\bX_N^u)}{h_0(X_0^u, 0)} \exp\biggr\{ 
    -\int_0^T \frac{u_t}{\sigma} \ud W_t - \int_0^T \frac{ \|u_t\|^2}{2 \sigma^2} \ud t \biggr\}  \biggr] \\
    \geq\;& \exp\biggr\{ \E\biggr[  \log \frac{g(\bX_{N}^u)}{h_0(X_0^u, 0)} 
    -\int_0^T \frac{u_t}{\sigma} \ud W_t -\int_0^T \frac{ \|u_t\|^2}{2 \sigma^2} \ud t   \biggr] \biggr\}, 
\end{align*}  
which yields that 
\begin{align*} 
  &\E\left[\int_0^T \frac{ \|u_t\|^2}{2 \sigma^2} \ud t   \right]
  \geq \E\left[  \log \frac{g(\bX_{N}^u)}{h_0(X_0^u, 0)}\right]=\E[\log g( \bX^u_N)] - \int \log h_0(x, 0) \mu_0(\ud x). 
\end{align*}  
The asserted result thus follows. 
\end{proof}

\section{Proof of the Existence of Solution to SSB} \label{sec:existence} 


We first recall the definition of the Hilbert metric~\citep{chen2016entropic}. 
For $K \subset \mathbb{R}^d$ and $1 \leq p \leq \infty$, let $\mathcal{L}^p(K)$ denote the $L^p$ space of functions defined on $K$. Define 
\begin{equation}\label{eq:def.L.plus}
\mathcal{L}_{+}^p(K) =  \{f \in \mathcal{L}^p(K) \colon  \inf_{x \in K} f(x) > 0 \}, 
\quad 
\mathcal{L}_{0}^p(K) =  \{f \in \mathcal{L}^p(K) \colon  \inf_{x \in K} f(x) \geq  0 \}. 
\end{equation} 
Since $\mathcal{L}_{0}^p(K)$ is a closed solid cone in the Banach space $\mathcal{L}^p(K)$, we can define a Hilbert metric on it. For any $x, y \in \mathcal{L}_{0}^p(K) \setminus \{0\}$ (where $0$ denotes the constant function equal to 0),  define 
\begin{equation}\label{eq:def.Mm.dH}
\begin{aligned}
M(x, y) =\inf \{c \colon x \preceq c y\},  \quad 
m(x, y) =\sup \{c \colon   c y \preceq x\}, 
\end{aligned}
\end{equation}
where $x \preceq y$ means $y - x \in \mathcal{L}_{0}^p(K)$ and we use the convention  $\inf \emptyset=\infty$. 
The Hilbert  metric  $d_H$ on $\mathcal{K} \setminus \{0\}$ is defined  by
\begin{equation}
d_H(x, y) = \log \frac{M(x, y)}{m(x, y)}. 
\end{equation}
Note that $d_H$ is only a pseudometric on $\mathcal{K} \setminus \{0\}$, but it is a metric  on the space of rays of $\mathcal{K} \setminus \{0\}$. 
 
\begin{proof}[Proof of Theorem~\ref{th:exist.compact}]
The proof is adapted from \citet[Proposition 1]{chen2016entropic}.  We will show the existence of strictly positive and integrable functions $\rho_0, \rho_T$ such that 
\begin{align}
    f_0(y) &= \rho_0(y) \int_{K_T} p(x, T \mid y, 0) \rho_T(x) \ud x,  \label{eq:equation1} \\
    f_T(x) &= \rho_T(x)^{(1 +\beta) / \beta} \int_{K_0} 
    p(x, T \mid y, 0) \rho_0(y) \ud y.  \label{eq:equation2}
\end{align}  
Let $\hat{\psi}_0 \in \mathcal{L}_{+}^\infty(K_0)$ be our guess for  $f_0  / \rho_0$. We can update $\hat{\psi}_0$ as follows.
\begin{enumerate}
    \item Set $\hat{\rho}_0(y) = f_0(y) / \hat{\psi}_0(y)$.
    \item Set $\hat{\psi}_T(x) = \left( \int_{K_0} 
    p(x, T \mid y, 0) \hat{\rho}_0(y) \ud y \right)^{\beta / (1 + \beta)}$, which is an estimate for $f_T^{\beta/(1 + \beta)} / \rho_T$ by~\eqref{eq:equation2}. 
    \item Set $\hat{\rho}_T(x) = f_T^{\beta/(1 + \beta)} / \hat{\psi}_T(x)$. 
    \item By~\eqref{eq:equation1}, update the estimate for $\psi_0$ by $\hat{\psi}_0^{\rm{new}}(y) =  \int_{K_T} p(x, T \mid y, 0) \hat{\rho}_T(x) \ud x$. 
\end{enumerate}
Denote this updating scheme by $ \hat{\psi}_0^{\rm{new}} = \mathcal{O}(\hat{\psi}_0)$.  
Note that for any $c \in (0, \infty)$ and $\psi \in \mathcal{L}_{+}^\infty(K_0)$, 
\begin{equation}\label{eq:scaling.O}
    \mathcal{O} (c \psi) = c^{\beta/(1 + \beta)} \mathcal{O} (\psi). 
\end{equation}
In Lemma~\ref{lm:O} below, we prove that $\mathcal{O}$ is a   strict contraction mapping from $\mathcal{L}_{+}^\infty(K_0)$ to $\mathcal{L}_{+}^\infty(K_0)$ with respect to the Hilbert metric. 
To prove the existence of $\rho_0, \rho_T$, it is sufficient to show that  $\mathcal{O}$ has a fixed point $\psi_0 \in \mathcal{L}_{+}^\infty(K_0)$, since we can  set 
\begin{align*}
    \rho_0(y) = \frac{f_0(y)}{\psi_0(y)}, \quad  \rho_T(x) = \left(\frac{f_T(x)}{\int_{K_0} p(x, T\mid y, 0) \rho_0(y) \, \ud y}\right)^{\frac{\beta}{1+\beta}}, 
\end{align*} 
which must satisfy~\eqref{eq:equation1} and~\eqref{eq:equation2} and $\rho_0 \in \mathcal{L}_0^1 (K_0), \rho_T \in \mathcal{L}_0^1 (K_T)$ (see the proof of Lemma~\ref{lm:O} for why $\rho_0, \rho_T$ are integrable). But note that we cannot apply Banach fixed-point theorem to $\mathcal{O}$. 

To find the fixed point of $\mathcal{O}$, we first consider its normalized version $\tilde{\mathcal{O}}$, defined by $\tilde{\mathcal{O}} (\psi) =  \mathcal{O}  (\psi) / \|  \mathcal{O} (\psi ) \|_2$. Let 
\begin{equation} 
E= \{g \in \mathcal{L}_{+}^\infty(K_0) \colon \|g \|_2=1 \}, 
\end{equation}
denote the domain and range of $\tilde{\mathcal{O}}$. 
Since $\mathcal{O}$ is a strict contraction mapping on $ \mathcal{L}_{+}^\infty(K_0)$ with respect to the Hilbert metric (which is invariant under scaling), $\tilde{\mathcal{O}}$ is also a strict contraction mapping and thus continuous (with respect to the Hilbert metric) on $E$.   
If $g \in \mathcal{L}_{+}^\infty(K_0)$ is a fixed point of $\tilde{O}$, then  $\|\mathcal{O}(g)\|_2^{1+\beta} g \in \mathcal{L}_{+}^\infty(K_0)$ is a fixed point of $\mathcal{O}$, since  
\begin{align*}
\mathcal{O}\left(\|\mathcal{O}(g)\|_2^{1+\beta} g\right)=\|\mathcal{O}(g)\|_2^{\beta}  \mathcal{O}(g)
= \|\mathcal{O}(g)\|_2^{1 + \beta}  \tilde{\mathcal{O}}(g) 
= \|\mathcal{O}(g)\|_2^{1 + \beta} g, 
\end{align*}
where the first equality follows from~\eqref{eq:scaling.O}. 
Moreover, since $d_H$ is a metric on the rays, any other fixed point of $\mathcal{O}$ must have the form $c g$ for some constant $c > 0$. But the same argument shows that $c$ must equal  $\|\mathcal{O}(g)\|_2^{1+\beta}$, and thus the fixed point of $\mathcal{O}$ is unique. 
This further yields the uniqueness of $(\rho_0, \rho_T)$. 

So it only remains to prove that $\tilde{O}$ has a fixed point in 
$\mathcal{L}_{+}^\infty(K_0)$.  
The proof for this claim is essentially the same as that in~\citet{chen2016entropic}. 
Pick arbitrarily $\psi^{(0)} \in \mathcal{L}_{+}^\infty(K_0)$ and let $g_0 = \psi^{(0)} / \| \psi^{(0)} \|_2$.
For $k = 1, 2, \dots$,  define 
\begin{equation}
    g_k \coloneqq \frac{\mathcal{O}^k (\psi^{(0)})}{\left\|\mathcal{O}^k (\psi^{(0)})\right\|_2} = \tilde{\mathcal{O}} (g_{k - 1}), 
\end{equation}
where the second equality follows from~\eqref{eq:scaling.O}.  
Each $g_k$ is in $E$ and well-defined as $\mathcal{O}^k (\psi^{(0)}) \in \mathcal{L}_{+}^\infty(K_0) \subset \mathcal{L}_{2}(K_0)$.  
Since $\tilde{\mathcal{O}}$ is  a strict contraction mapping with respect to $d_H$, $\{g_k\}_{k \geq 0}$ is a Cauchy sequence with respect to $d_H$. 
Using the inequality $\| g_k - g_l \|_2 \leq e^{d_H(g_k, g_l)}- 1$ (see~\citet{chen2016entropic}), we find that $\{g_k\}_{k \geq 0}$ is also Cauchy with respect to the $L^2$-norm, and thus there exists $g \in \mathcal{L}^2_0(K_0)$ such that $\lim_{k \rightarrow \infty} \| g_k - g\|_2 = 0$ and $\| g \|_2 = 1$. 
Next, we argue that $\{g_k\}_{ k \geq 0}$ is uniformly bounded from below and above and also uniformly equicontinuous. 
To show the uniform boundedness, we first observe that $\| g_k \|_2 = 1$ implies 
\begin{equation}\label{eq:unif.bound.1}
   \sup_x g_k(x) \geq \frac{1}{\sqrt{ \lambda(K_0) }} \geq \inf_x g_k(x). 
\end{equation}
Further, since $p(x, T \mid y, 0)$ is bounded in $(x, y)$ and recalling the last step in the construction of $\mathcal{O}$, 
we have 
\begin{equation}\label{eq:unif.bound.2}
\frac{\sup_x (\mathcal{O} \psi)(x)}{\inf_x (\mathcal{O} \psi)(x) } \leq \epsilon^{-1} 
\end{equation}
for any $\psi \in  \mathcal{L}_{+}^\infty(K_0)$, where $\epsilon \in (0, 1]$  is some constant independent of $\psi$. Combining~\eqref{eq:unif.bound.1} and~\eqref{eq:unif.bound.2}, we get  
\begin{equation}
\frac{\epsilon}{ \sqrt{\lambda ( K_0) }} \leqslant \inf_x g_k(x) \leqslant \sup_x g_k(x) \leqslant \frac{1}{\epsilon \sqrt{\lambda( K_0) }}\label{eq:equation3}. 
\end{equation}
Note that the uniform boundedness of $\{g_k\}$ also implies $g \in \mathcal{L}^\infty_+ (K_0)$.  
The uniform equicontinuity of $\{g_k\}$ can be proved by using the uniform continuity of the transition density $p$. 
Finally, by Arzel\`{a}–Ascoli Theorem, there is a subsequence $\{g_{k_n}\}$ such that $g_{k_n}$ converges to $g$ uniformly with respect to the $L^2$-norm and $g$ is also uniformly continuous. 
This implies $d_H( g_{k_n}, g) \rightarrow 0$ and thus $d_H(g_k, g) \rightarrow 0$. By the continuity (with respect to $d_H$) of $\tilde{O}$, we can interchange the limit operation with $\tilde{\mathcal{O}}$, thereby establishing $g$ as the fixed point of $\tilde{\mathcal{O}}$.  
\end{proof}

\begin{lemma}\label{lm:O}
For $\psi \in \mathcal{L}_{+}^\infty(K_0)$, define 
\begin{equation}
    (\mathcal{O}\psi)(x) =    \int_{K_T} p(z, T\mid x, 0) \left( \frac{ f_T(z) }{ 
    \int_{K_0} p(z, T\mid y, 0) f_0(y) \psi(y)^{-1} \ud y
    } \right)^{\beta / (1+\beta) } \ud z.
\end{equation}
The operator $\mathcal{O}$ is a strict contraction mapping from $\mathcal{L}_{+}^\infty(K_0)$ to 
$\mathcal{L}_{+}^\infty(K_0)$ with respect to the Hilbert metric. 
\end{lemma} 
\begin{proof}
We can express the operator $\mathcal{O}$ by 
    $\mathcal{O} = \mathcal{E}_{T}\circ\mathcal{P}\circ\mathcal{I}\circ\mathcal{E}_{0}\circ\mathcal{I}$, 
and define $\mathcal{I}, \mathcal{E}_0, \mathcal{E}_T, \mathcal{P}$ by  \begin{equation}
\begin{aligned}
& \mathcal{I}\colon \mathcal{L}_{+}^\infty(K) \longrightarrow \mathcal{L}_{+}^\infty(K),  \quad \quad \; (\mathcal{I} \varphi)(x) = \varphi(x)^{-1}   \quad \text{for }  K = K_0, K_T, \\
& \mathcal{E}_{0}\colon \mathcal{L}_{+}^\infty(K_0) \longrightarrow \mathcal{L}_{+}^\infty(K_T), \quad   (\mathcal{E}_{0} \varphi)(x)  =   \int_{K_0} p(x, T\mid y, 0) f_0(y) \varphi(y) \ud y, \\
& \mathcal{E}_{T}\colon \mathcal{L}_{+}^\infty(K_T) \longrightarrow \mathcal{L}_{+}^\infty(K_0), \quad   (\mathcal{E}_{T} \varphi)(x) =  \int_{K_T} p(z, T\mid x, 0) f_T(z)^{\beta / (1+\beta) } \varphi(z) \ud z, \\
& \mathcal{P}\colon \mathcal{L}_{+}^\infty(K_T) \longrightarrow \mathcal{L}_{+}^\infty(K_T),  \quad  (\mathcal{P} \varphi)(x) = \varphi(x)^{\beta / 1+\beta}. 
\end{aligned}
\end{equation}  
It is worth explaining how the ranges of these operators are determined. First, if $\varphi \in \mathcal{L}_{+}^\infty(K) $, it is clear that $\mathcal{I} \varphi \in \mathcal{L}_{+}^\infty(K)$ and $\mathcal{P} \varphi \in \mathcal{L}_{+}^\infty(K)$. 
For $\mathcal{E}_0$, since we assume $K_0$ is compact and $p(x, T \mid y, 0)$ is continuous in $(x, y)$, 
for any $\varphi \in \mathcal{L}_{+}^\infty(K_0)$ there exists $\epsilon > 0$ such that 
\begin{align*}
 \epsilon =  \epsilon \int_{K_0}  f_0(y) \ud y \leq (\mathcal{E}_{0} \varphi)(x)  \leq \epsilon^{-1} \int_{K_0}  f_0(y) \ud y = \epsilon^{-1}. 
\end{align*}
The argument for $\mathcal{E}_T$ is similar, and note that by H\"{o}lder's inequality,
\begin{align*}
   \int_{K_T}  f_T(z)^{\beta/(1 + \beta)} \ud z  \leq
   \lambda(K_T)^{1 / (1 + \beta)} < \infty. 
\end{align*}
Now we prove that $\mathcal{P}$ is a strict contraction. 
Let $\psi_1, \psi_2 \in \mathcal{L}_{+}^\infty(K_T)$.  By the definition given in~\eqref{eq:def.Mm.dH}, $$M\left(\mathcal{P}\left(\psi_1\right), \mathcal{P}\left(\psi_2\right)\right)=M\left(\psi_1, \psi_2\right)^{\beta / 1+\beta}  \text{and} \quad 
m\left(\mathcal{P}\left(\phi_1\right), \mathcal{P}\left(\phi_2\right)\right)=m\left(\phi_1, \phi_2\right)^{\beta / 1+\beta},$$ 
which implies 
\begin{equation}
 d_H\left(\mathcal{P}\left(\phi_1\right), \mathcal{P} \left(\phi_2\right)\right)=\log \left(\frac{M\left(\mathcal{P}\left(\phi_1\right), \mathcal{P}\left(\phi_2\right)\right)}{m\left(\mathcal{P}\left(\phi_1\right), \mathcal{P}\left(\phi_2\right)\right)}\right) = \frac{\beta}{1+\beta} d_H\left(\phi_1, \phi_2\right) <  d_H\left(\phi_1, \phi_2\right), 
\end{equation}
since $\beta < \infty$.  
\cite{chen2016entropic} showed that the operators $\mathcal{E}_{0}, \mathcal{E}_{T}$ are strict contractions using Birkhoff's theorem, and that the operator $\mathcal{I}$ is an isometry (all with respect to the Hilbert metric). Hence, $\mathcal{O}$ is a strict contraction. 
Note that for our problem, since $\mathcal{P}$ is a strict contraction, we actually only need $\mathcal{E}_{0}, \mathcal{E}_{T}$ to be contractions (not necessarily strict). 
\end{proof}

\section{Proof  of Lemma~\ref{lm:two.SB}}\label{sec:proof.two.SB}

We consider a more general setting. Assume that $X_t$ is given by 
\begin{equation}\label{eq:X2}
  X_0 = x_0, \text{ and }  \ud X_t= b(X_t, t) \ud t + \sigma \ud W_t, 
\end{equation}
where $b$ satisfies Assumption~\ref{ass.b}. 
The density function of $\Law(X_T)$ is $p(x, T \mid x_0, 0)$. 
Let $X^{\REF}$ be given by 
\begin{align} 
    \ud X_t^{\REF} =\;& b^{\REF}(X_t^{\REF}, t)   \ud t + \sigma \ud W_t,  \label{eq:Xref2} \\
    \text{ where }   
    b^{\REF}(x, t) = \;& b(x, t) + \sigma^2 \nabla  \log h^{\REF}(x, t), \\ 
    h^{\REF}(x, t) = \;& \E\left[  \frac{  f_{\REF}(X_T)  }{ p(X_T, T \mid x_0, 0)  }  \,\Big|\, X_t = x \right]. 
\end{align}
Theorem~\ref{th:dp1} implies that $\Law(X_T^{\REF}) = \muR$. 
Let $X^{\TGT}$ be given by 
\begin{align}
    \ud X^{\TGT}_t =\;&   b^{\TGT}(X_t^{\TGT}, t)   \ud t + \sigma \ud W_t, \label{eq:Xtgt2} \\ 
\text{ where }  
    b^{\TGT}(x, t) = \;& b^{\REF}(x, t) + \sigma^2 \nabla  \log h^{\TGT}(x, t), \\ 
h^{\TGT}(x, t) = \;& \E\left[  \left( \frac{ f_{\TGT}(X_T^{\REF}) }{ f_{\REF}(X_T^{\REF}) } \right)^{\beta/(1+\beta)} \,\Big|\, X_t^{\REF} = x \right], \label{eq:htgt}
\end{align}
which is the solution to Problem~\ref{problem1} with $X^{\REF}$ being the reference process and $\muT$ being the target distribution. 
We now prove that 
\begin{equation}\label{eq:goal.two.SB}
b^{\TGT}(x, t)    =  b(x, t) + \sigma^2 \nabla \log \E\left[    \frac{  f_{\REF}(X_T)^{ \frac{1}{1+\beta}}  f_{\TGT}(X_T)^{\frac{\beta}{1+\beta}} }{ p(X_T, T \mid x_0, 0)  } \,\Big|\, X_t  = x \right].
\end{equation} 
That is, $X^{\TGT}$ is also the solution to Problem~\ref{problem0} with $X$ being the reference process and $\mu^*$ being the target distribution, where $\mu^*$ has un-normalized density $f_{\REF}^{1/(1+\beta)}  f_{\TGT}^{ \beta/(1+\beta)}$. Once this is proved, Lemma~\ref{lm:two.SB}  follows as a  special case with $b \equiv 0$ and $x_0 = 0$. 

\begin{proof} 
To prove~\eqref{eq:goal.two.SB}, we use part (ii) of Theorem~\ref{th:h}. Define a martingale 
\begin{align*}
   Z^{\REF}_t = \frac{h^{\REF}(X_t, t) }{h^{\REF}(X_0, 0)}. 
\end{align*}
Let $\Q^{\REF}$ be the probability measure given by $\ud \Q^{\REF} = Z^{\REF}_T \ud \P$. By part (ii) of Theorem~\ref{th:h}, the law of $X$ under $\Q^{\REF}$ is the same as the law of $X^{\REF}$ under $\P$. Applying the change of measure to~\eqref{eq:htgt} yields 
\begin{align*}
    h^{\TGT}(x, t) 
    = \;&  \E\left[  \left( \frac{ f_{\TGT}(X_T^{\REF}) }{ f_{\REF}(X_T^{\REF}) } \right)^{\beta/(1+\beta)} \frac{Z_T^{\REF}}{Z_t^{\REF}}    \,\Big|\, X_t = x \right], \\
    = \;& h^{\REF}(X_t, t)^{-1}\E\left[  \left( \frac{ f_{\TGT}(X_T^{\REF}) }{ f_{\REF}(X_T^{\REF}) } \right)^{\beta/(1+\beta)} h^{\REF}(X_T, T)  \,\Big|\, X_t = x \right], \\
    =\;& h^{\REF}(X_t, t)^{-1}\E\left[   \frac{  f_{\REF}(X_T)^{ \frac{1}{1+\beta}}  f_{\TGT}(X_T)^{\frac{\beta}{1+\beta}} }{ p(X_T, T \mid x_0, 0)  }  \,\Big|\, X_t = x \right].
\end{align*}
Since
\begin{align*}
    b^{\TGT}(x, t) &= b^{\REF}(x, t) + \sigma^2 \nabla  \log h^{\TGT}(x, t)  \\
    &= b(x, t) + \sigma^2 \nabla  \log h^{\REF}(x, t) 
    + \sigma^2 \nabla  \log h^{\TGT}(x, t),  
\end{align*}
we obtain~\eqref{eq:goal.two.SB}. 
\end{proof}

\section{MNIST Example} \label{sec:mnist.supp} 
Figure~\ref{fig:DT} visualizes the 50 images in the data set $\DT$, which are obtained by adding Gaussian noise to the original images in MNIST.
Figure~\ref{fig:path2} shows the new images generated by the two-stage Schr\"{o}dinger bridge  algorithm of~\citet{pmlr-v139-wang21l} using only $\DT$ as the input.  

Table~\ref{tabel:IS} shows the inception scores~\citep{salimans2016improved} for our generated images and the images of digit 8 in MNIST. 
The score of our samples for $\beta = 1.5$ is slightly higher than that of the digit 8 in MNIST dataset, suggesting that our generated images of digit 8 exhibit a greater degree of variability than those in MNIST. 
Additionally, when $\beta = 100$, our score aligns closely with that of $\DT$ (i.e., the noisy digit 8 images from MNIST), 
indicating that our method can recover the images in the target data set by using a large $\beta$. 
For small values of $\beta$, the scores of our generated images are higher than that of digit 8 images in MNIST,  primarily because the reference dataset (consisting of the other digits) has greater variability and complexity. However, as shown in Figure~\ref{fig:path1}, when $\beta$ is small, we do not necessarily get images of digit 8. 

We also utilize t-SNE plots to visually characterize the distribution of our generated images. Figure~\ref{fig:tsne-3}  illustrates that our samples come from the geometric mixture distribution interpolating between  the noisy images of digit 8 and the clean images of other digits. Figure~\ref{fig:tsne-4} demonstrates that SSB samples with $\beta = 1.5 $ are positioned closer to the clean images of digit 8 compared to the samples obtained with $\beta = 100$.  

In our code, we use the neural network model of~\citet{song2019generative} for training the score functions and use the neural network model of~\citet{pmlr-v139-wang21l} for training the density ratio function.

\begin{minipage}{0.45\linewidth}
\begin{figure}[H]
    \centering
    \includegraphics[width=0.9\linewidth]{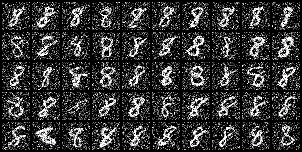} 
    \caption{Samples in $\DT$.}
    \label{fig:DT}
\end{figure}
\end{minipage}\hspace{1cm}
\begin{minipage}{0.52\linewidth}
\begin{figure}[H]
    \centering 
    \includegraphics[width=0.9\linewidth]{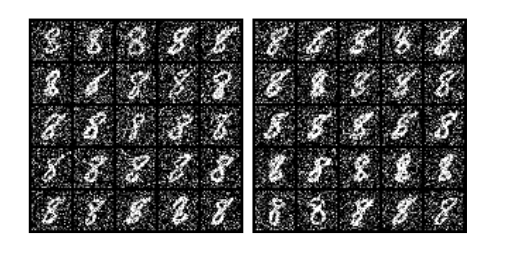}
    \caption{Samples Generated by the Algorithm of~\citet{pmlr-v139-wang21l} Using Only $\DT$.}
    \label{fig:path2}
\end{figure}
\end{minipage}

\begin{minipage}{0.45\linewidth}
\begin{figure}[H]
 \includegraphics[width=0.9\linewidth]{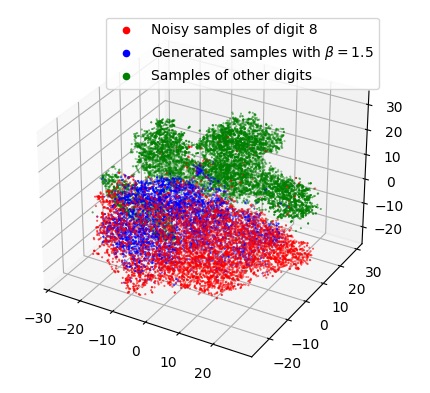}
    \caption{t-SNE Plot Illustrating the Geometric Mixture Distribution with $\beta = 1.5$.}
        \label{fig:tsne-3}
 \end{figure}
 \end{minipage} \hspace{2cm}
 \begin{minipage}{0.45\linewidth}
\begin{figure}[H]
       \includegraphics[width=0.9\linewidth]{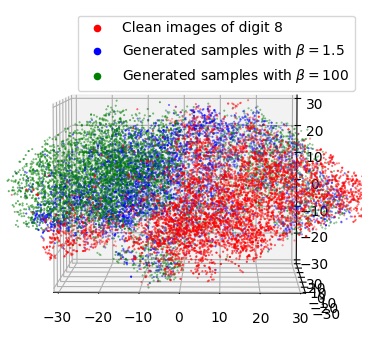}  
    \caption{t-SNE Plot Comparing SSB Samples with Clean Images in MNIST.  } \label{fig:tsne-4}
\end{figure}
\end{minipage} 
 
\begin{table}[H]
    \caption{Inception Scores (mean $\pm$ sd) for SSB Samples and the Images in MNIST.}\label{tabel:IS}
    \begin{center}
    \begin{tabular}{|l|c|}
      \hline
      Datasets (sample size $\approx 5$K) & Inception score   \\
      \hline
            SSB with $\beta =0$ &   6.70 $\pm$ 0.20 \\
            SSB with $\beta =0.25$ &   6.59 $\pm$ 0.15 \\
            SSB with $\beta =0.7$ &  5.12 $\pm$ 0.11\\
      SSB with $\beta =1.5$ &   3.51 $\pm$ 0.08 \\
      SSB with $\beta =4$ &  3.65 $\pm$ 0.04 \\  
      SSB with $\beta =100$ &  2.87 $\pm$ 0.04 \\
      Digit 8 in MNIST (clean) &   3.29 $\pm$ 0.04 \\
      Digit 8 in MNIST (noisy) & 2.96 $\pm$ 0.04\\
      \hline
    \end{tabular}
    \end{center}
\end{table}

\end{document}